%% file: Main.tex
\documentclass{article}

\usepackage[margin=1in, textwidth=6.5in, marginparwidth=0pt]{geometry}

\title{Lower Bounds on Adversarial Robustness for Multiclass Classification with General Loss Functions}
\author{\small Camilo Andrés García Trillos \\ \small Department of Mathematics, University College London \\ \small \texttt{camilo.garcia@ucl.ac.uk} \\ \and \small Nicolás García Trillos \\ \small Department of Statistics, University of Wisconsin Madison \\ \small \texttt{garciatrillo@wisc.edu}}

\usepackage{hyperref}
\hypersetup{
  pdftitle={Lower Bounds on Adversarial Robustness for Multiclass Classification with General Loss Functions},
  pdfauthor={C.A. García Trillos, and N. García Trillos},
  pdfstartview=FitH,
  pdfpagemode=UseOutlines,
  bookmarksopen=true,
  unicode=true
}

\input{preamble_arxiv}

\begin{document}

\maketitle

\begin{abstract}
We consider adversarially robust classification in a multiclass setting under arbitrary loss functions and derive dual and barycentric reformulations of the corresponding learner-agnostic robust risk minimization problem. We provide explicit characterizations for important cases such as the cross-entropy loss, loss functions with a power form, and the quadratic loss, extending in this way available results for the 0-1 loss. These reformulations enable efficient computation of sharp lower bounds for adversarial risks and facilitate the design of robust classifiers beyond the 0-1 loss setting. Our paper uncovers interesting connections between adversarial robustness, $\alpha$-fair packing problems, and generalized barycenter problems for arbitrary positive measures where Kullback-Leibler and Tsallis entropies are used as penalties. Our theoretical results are accompanied with illustrative numerical experiments where we obtain tighter lower bounds for adversarial risks with the cross-entropy loss function. 
\end{abstract}

\section{Introduction}

In this paper, we study a class of minmax problems of the form
\begin{equation}
  \min\limits_{f \in \F} \max_{ \tilde \mu \in \mathcal{P}(\Z)} R(\tilde\mu,f) - C(\mu , \tilde \mu),
  \label{eqn:min_max_problem}
\end{equation}
where $R$ is the risk functional 
\[ R(\tilde\mu,f ):= \int_{\X \times \Y} \ell(f(x), y) d\tilde{\mu}(\tilde x , \tilde y)  \]
associated to a loss functions $\ell$, and where $C$ is a cost function between pairs of probability distributions $\mu$ and $\tilde \mu$ over the product space $\Z=\X \times \Y$. Here and in the sequel, we will think of $\X$ as a feature space, which we assume has the structure of a Polish metric space with distance function $d$, and of $\Y$ as a (finite) set of labels. Problem \eqref{eqn:min_max_problem} can be interpreted as a two-player game, played between a learner and an adversary, that captures the learner's desire to build classification models that are robust against adversarial perturbations of a clean data distribution, here represented by $\mu$. In this interpretation, the set $\F$ in \eqref{01_barycenter} is a family of soft classification models (i.e., measurable maps $f: \X \rightarrow \Delta_{\Y}$, for $\Delta_\Y$ the probability simplex over $\Y$) accessible to the learner, and $\tilde \mu$ is the new data distribution that is selected by the adversary. $C(\mu, \tilde \mu)$ is the cost that the adversary must pay to modify the clean data distribution $\mu$ and rearrange it as $\tilde \mu$. This function implicitly determines the types of attacks that are feasible for the adversary.

Throughout the paper, we will mostly focus on a special and important choice for the cost function $C$ and the family of classification models $\F$. First, we assume that $C$ has the structure of an optimal transport problem
\begin{equation}
C(\mu, \tilde \mu) := \inf_{ \pi \in \Gamma(\mu, \tilde \mu)       }  \int_{\Z\times \Z} c_\Z( (x,y), (\tilde x , \tilde y) ) d\pi((x,y), (\tilde x , \tilde y))   
\label{def: cost type}
\end{equation}
for a marginal cost function $c_\Z : \Z \times \Z \rightarrow \R_+ \cup \{ \infty\}$ satisfying
\begin{equation}
c_\Z((x,y), (\tilde x , \tilde y)):= \begin{cases} c(x, \tilde x), & \text{ if } y =\tilde y , \\ \infty, &  \text{else.} \end{cases}   
\label{eqn:MarginalCost}
\end{equation}
This specific form for $c_\Z$ forces the adversary to respect labels when perturbing arbitrary data points. In mathematical terms,

under this cost function problem
\ref{eqn:min_max_problem} can be rewritten as 
\begin{equation}
    \inf_{f \in \F} \sup _{ \{\tilde{\mu}_i\}_{i \in \mc Y}  } \sum_{i \in \Y}   \int_\X \ell(f(\tilde x), i) d\tilde \mu_i(\tilde x)   - \sum_{i \in \Y} C(\mu_i, \tilde \mu_i),
    \label{eqn:ATGeneralLoss}
\end{equation}
where in the above and in the sequel we abuse the notation introduced earlier and set 
\[ C(\mu_i, \tilde \mu_i) := \inf_{\pi_i \in \Gamma(\mu_i, \tilde \mu_i)} \int_{\X \times \X} c(x, \tilde x) d\pi_i(x, \tilde x);\]
here, for a fixed $i \in \Y$ we use $\mu_i$ (or $\tilde \mu_i$) to denote the positive measures over $\X$ (not necessarily normalized) defined as $\mu_i(A) = \mu(A \times \{ i\})$ for $A$ a (Borel) measurable subset of $\X$ (note that $\sum_{i\in \Y} \mu_i(\X)=1$). As an example of the types of cost function $c: \X \times \X  \mapsto \R_+\cup \{ \infty\}$ that are of interest in the literature, we may consider the  \textit{$0$-$\infty$ cost} given by
\begin{equation}
  c_{\veps}(x, \tilde x):= \begin{cases} 0 \quad \text{ if } d(x, \tilde x) \leq \veps \\ \infty \quad 
  \text{else},\end{cases}  
  \label{eqn:0inftyCost}
\end{equation}
 for $\veps$ a positive parameter often referred in the literature as \textit{adversarial budget}. In this context, $\veps$ represents the maximum size of data perturbations that the adversary may deploy around any given clean data point. For this cost function, problem \eqref{eqn:ATGeneralLoss} can be seen to reduce to
\begin{equation}
\inf_{f \in \F}   \sum_{i \in \Y}   \int_\X  \sup_{\tilde x \in B_\veps(x)}\ell(f(\tilde x), i) d\mu_i(x),
\label{eqn:StnadardAT}    
\end{equation}
a model that in the literature is known as \textit{adversarial training}; see Appendix \ref{app:AT} for some informal discussion of this equivalence.


Regarding the family of classification models $\F$, we will focus on the \textit{agnostic-learner} setting, which corresponds to the choice $\F= \F_{\mathrm{all}}$ given by
\begin{equation}
 \mathcal{F}_{\mathrm{all}} := \{  f : \X \rightarrow \Delta_Y \text{ Borel}\}.  
\end{equation}
In words, $\F_{\mathrm{all}}$ is the set of \textit{all} measurable soft classifiers from the feature space $\X$ into the set of labels $\Y$. In addition to being important for theoretical reasons (e.g., a minimizer of the agnostic \textit{robust} risk minimization problem can be interpreted as a \textit{robust} Bayes classifier), when we select $\F = \F_{\mathrm{all}}$ in \eqref{eqn:ATGeneralLoss} we obtain a fundamental lower bound for the value of problem \eqref{eqn:ATGeneralLoss} with \textit{any other subfamily} ${\F}$ of measurable soft classifiers. Precisely, we have
\begin{align}
\begin{split}
\inf_{f \in \F_{\mathrm{all}}} \sup _{ \{\tilde{\mu}_i\}_{i \in \mc Y}  } & \sum_{i \in \Y}   \int_\X \ell(f(\tilde x), i) d\tilde \mu_i(\tilde x)   - \sum_{i \in \Y} C(\mu_i, \tilde \mu_i) 
\\ & \leq  \inf_{f \in \F} \sup _{ \{\tilde{\mu}_i\}_{i \in \mc Y}  } \sum_{i \in \Y}   \int_\X \ell(f(\tilde x), i) d\tilde \mu_i(\tilde x)   - \sum_{i \in \Y} C(\mu_i, \tilde \mu_i), 
         \end{split}
           \label{eqn:LowerBound}
\end{align}
regardless of the choice of $\F$ (families of neural networks, kernel machines, etc). This lower bound, which, as we suggest throughout the paper, can be computed more efficiently than the right-hand side of \eqref{eqn:LowerBound}, is a useful benchmark for training robust learning models in practical settings, when $\mathcal{F}$ is typically assumed to be some rich parametric family of classifiers.
\medskip

For a cost function $C$ as above and for the family of learning models $\F= \F_{\mathrm{all}}$, if the loss function $\ell: \Delta_\Y \times \Y \rightarrow \R$ is chosen to be the \textit{$0$-$1$ loss} defined as 
\begin{equation}
   \ell_{01}(v , i) :=  1- v_i, \quad i \in  \Y , \quad v \in \Delta_\Y, \footnote{This linear function is a natural extension of the standard 0-1 loss for hard classifiers to soft classifiers, and we thus refer to it as 0-1 loss.} 
   \label{eqn:01}
\end{equation}
it has been shown in \cite{MOTJakwang} that the agnostic-learner version of \eqref{eqn:ATGeneralLoss} (i.e., the case $\F= \F_{\mathrm{all}}$) is equivalent to the optimization problem
   \begin{equation}
    \begin{aligned}
    &\sup_{ \{ g_i \}_{i \in \Y}  }& \quad  &  \sum_{i \in \Y} \int_\X g_i(x_i) d \mu_i (x_i), \\
    & \qquad \mathrm{s.t.}& & \sum_{i\in A} g_i(x_i) \leq   1 +     c_A(x_A), \quad \forall x_A \in \spt(\mu_A) , \, \forall A \subseteq \Y,
    \label{dual_0-1}
      \end{aligned}
      \end{equation}
where 
\[ c_A(x_A) := \inf_{\tilde x \in \X} \sum_{i \in A} c(x_i, \tilde x ) \]
and $\spt(\mu_A)$ denotes the support of the product measure $\otimes_{i \in A}\mu_i$. Another equivalent reformulation of \eqref{eqn:ATGeneralLoss}, in the form of a \textit{generalized barycenter problem} for the measures $\{\mu_i \}_{i \in \Y}$, was also derived in \cite{MOTJakwang}:
\begin{equation}
    \inf_{\lambda, \{\tilde \mu_i\}_{i \in \mc Y}  } \left\{ \lambda (\mc X)  + \sum_{i\in \mc Y} C(\mu_i, \tilde \mu_i) :  \tilde \mu_i \leq \lambda \text{ for all } i \in \mc Y \right\}.
    \label{01_barycenter}
\end{equation}
Here, the inf ranges over collections of finite positive measures over $\X$, and the constraint $\tilde{\mu}_i \leq \lambda$ is understood in the sense of measures (i.e., $\tilde{\mu}_i(A) \leq \lambda(A)$ for all Borel measurable $A \subseteq \X$). Precisely, \cite{MOTJakwang} shows that the infimum in problem \eqref{eqn:ATGeneralLoss} with $\F= \F_{\mathrm{all}}$ and $\ell= \ell_{01}$ is equal to $1 -\eqref{dual_0-1} = 1 - \eqref{01_barycenter}$. These equivalent reformulations of problem \eqref{eqn:ATGeneralLoss} for the $0$-$1$ loss have facilitated the development of computational algorithms to obtain lower bounds for the adversaria risk of arbitrary models trained with this loss function. These methods exploit the aforementioned equivalences and in particular take advantage of the many tools in the literature of computational optimal transport that have been developed in the past decade; see the discussion in section \ref{subsec:Lowerbounds} below and in \cite{ATThroughOT,NEURIPS2023_9b867f0e,PenkaGenCol}. It is not surprising that optimal transport techniques can be used to solve these problems since, after all, problem \eqref{dual_0-1} is the dual of a problem closely related to multimarginal optimal transport (MMOT) and \eqref{01_barycenter} has the form of a generalized barycenter problem in the space of positive measures. Further, solutions of such problems can be leveraged to recover \textit{optimal (agnostic) robust classifiers}. Indeed, we can use the solution to \eqref{dual_0-1} to construct a solution $f^*:\X \rightarrow \Delta_\Y$ to problem \eqref{eqn:ATGeneralLoss} by using the formula
\begin{equation}
   f^*_i(\tilde x) = \max \{ - g_i^c(\tilde x) ,0    \} , \quad i \in \Y, \footnote{The Borel measurability of this function depends on the cost function $c$. It is guaranteed, for example, when the cost function $c$ is continuous. Some care must be taken when considering cost functions like $c_\veps$ in \eqref{eqn:0inftyCost}; see \cite{ExistenceSolutionsAT} for a discussion of these measurability issues.}
   \label{eq:RobustClassif0-1}
\end{equation}
where 
\begin{equation}  
g_i^c(\tilde x) := \inf_{x \in \text{spt}(\mu_i)} \{ c(x, \tilde x) -  g_i(x) \}  
\label{def c-transform}
\end{equation}
is the so-called $c$-transform of $g_i$ \footnote{The notion of $c$-transform considered in this paper uses the infimum over the support of the measures $\mu_i$ only. In particular, when the $\mu_i$ are concentrated over finitely many points, \eqref{def c-transform} optimizes over finitely many $x$ and only the values of $g_i$ at those points are important for the definition of $g_i^c$ at an arbitrary $\tilde x$.}. We reiterate that, given the agnostic nature of the problem we have posed, \eqref{eq:RobustClassif0-1} produces the minimal (robust) risk achievable by \textit{any} classifier when the risk used to quantify data mismatch is the one associated to the $0$-$1$ loss, in accordance with \eqref{eqn:LowerBound}.

\medskip 

Although the above is a compelling story on how to study adversarial robustness through the lens of theoretical and computational tools in optimal transport, this rich framework has been restricted, to our knowledge, to the 0-1 loss setting described above. In particular, there has not been much discussion on how to compute \textit{sharp} agnostic lower bounds like \eqref{eqn:LowerBound} for more general loss functions $\ell$ (such as the cross-entropy), despite the fact that there are more popular loss functions used in practical settings than the $0$-$1$ loss. Our goal in this paper is to fill this gap and develop analogous results for more general loss functions.

Obtaining analogous results for the cross-entropy loss function was one of the main motivations for this paper, given that the majority of training routines used in data science are performed under this loss function. However, our analysis will allow us to cover other important and interesting cases. Indeed, through our analysis we will reveal interesting connections between the adversarial model \eqref{eqn:ATGeneralLoss} for quite general loss functions $\ell$, a problem in the optimization literature known as $\alpha$-fair packing (see Appendix \ref{app:AlphaFair}), and generalizations of the barycenter problem in spaces of measures appearing in \eqref{01_barycenter} that use \textit{Tsallis entropies} to relax the hard constraints in \eqref{01_barycenter}. These connections, in turn, open the door to the use of a wide range of optimization tools to solve the adversarial problem \eqref{eqn:ATGeneralLoss}. As an application of our main results, in the final section of our paper we obtain sharper lower bounds for adversarial training (AT) with the cross-entropy loss function in simple practical settings, which is a significant extension of the results in \cite{MOTJakwang} and \cite{ATThroughOT}. Indeed, as discussed above, the results and experiments in those papers were restricted to the 0-1 loss case. We compare the lower bounds obtained for the 0-1 and cross-entropy loss functions, illustrating the gain of obtaining sharper lower bounds for the adversarial risk of models trained with the cross-entropy loss. This discussion is presented in section \ref{subsec:Lowerbounds} below.

\subsection{Main results}
\label{sec:IntroMianResults}

Our first result deduces an equivalent formulation for problem \eqref{eqn:ATGeneralLoss} that is analogous to \eqref{dual_0-1} but that applies for quite general convex loss functions $\ell$. The precise assumptions that we impose on the loss function $\ell$ and the cost function $c$ are presented next.

\begin{assumption}
\label{assump:LossFunction}
    We assume that, for every $i \in \Y$, the function $\ell(\cdot, i) : \Delta_\Y \rightarrow \R_+\cup  \{ \infty \}$ is convex. Also, we assume that there is $v_0 \in \Delta_\Y$ such that $\ell(v_0, i)\not = \infty$ for all $i \in \Y$. 

\end{assumption}

\begin{assumption}
The function $c: \X \times \X \rightarrow \R_+\cup \{ \infty\}$ is assumed to be lower-semicontinuous and to satisfy $c(x,x) =0 $ for all $x \in \X$. Furthermore, we assume one of the following two conditions:
\begin{enumerate}
\item $c$ satisfies the following compactness and coercivity condition: if $\{ \tilde x_n \}_{n \in \N}$ is a bounded sequence in $\X$ and 
$\{  x_n\}_{\in \N}$ is another sequence for which $\sup_{n \in \N} c(x_n, \tilde x_n)<\infty$, then $\{(x_n, \tilde x_n)\}_{n \in \N}$ is precompact in $\X \times \X$ (with the product topology). 
\item $c$ is of the form $c= \min\{ c_0,B\}$ for some scalar $B >0$ and some cost function $c_0 \geq 0$ satisfying the above compactness and coercivity condition.
\end{enumerate}
\label{assump:Cost}
\end{assumption}

Note that Assumption \ref{assump:Cost} on the cost function $c$ is the same as in \cite{MOTJakwang}. As discussed there, for the cost function \eqref{eqn:0inftyCost} to satisfy Assumptions \ref{assump:Cost}, $\X$ needs to be assumed to be a locally compact space (e.g., Euclidean space, or a finite dimensional manifold). 

We are ready to present our first main result.

\begin{theorem}
\label{thm:main}
Under Assumption \ref{assump:LossFunction} on the loss function $\ell$ and Assumption \ref{assump:Cost} on the cost function $c$, problem \eqref{eqn:ATGeneralLoss} with $\F = \F_{\mathrm{all}}$ has the same value as the problem
 \begin{equation}
    \begin{aligned}
    &\inf_{ \{ \phi_i \}_{i \in \Y}  \subseteq \G  }& \quad  & - \sum_{i \in \Y} \int_\X \phi_i(x_i) d \mu_i (x_i), \\
    & \qquad \mathrm{s.t.}& & 0 \geq \sup_{m_A \in \Delta_{A}} \left\{ \sum_{i \in A} m_i\phi_i(x_i)  +  \ell_A(m_A) - c_A(x_A , m_A)      \right\}, \, \forall x_A \in \spt(\mu_A),\, \forall A \subseteq \Y,
    \label{dual form}
      \end{aligned}
      \end{equation}

when $\G$ is taken to be $C_b(\X)$, the space of bounded continuous functions on $\X$. Here and in the sequel, we use $\supp(\mu_A)$ to denote the support of the product measure $\mu_A=\otimes_{i \in A} \mu_i $ and $\Delta_A$ to represent the probability simplex on the elements of the subset $A$ of $\Y$. The scalar functions $\ell_A$ and $c_A$ are defined according to
\[ \ell_A(m_A):=\inf_{v \in \Delta_\Y} \sum_{i \in A} \ell(v,i) m_i,   \qquad c_A(x_A, m_A):= \inf_{\tilde x\in \X } \sum_{i \in A} c(x_i, \tilde x ) m_i. \]

Moreover, if $\{\phi_i^*\}_{i \in \Y}$ is a solution to problem \eqref{dual form} for a set $\G$ containing $C_b(\X)$, then a Borel measurable function $f^*:\X \rightarrow \Delta_\Y $ satisfying
\begin{equation}
f^*(\tilde x ) \in  \arg \min_{v \in \Delta_\Y}  \max_{m \in \Delta_\Y} \sum_{i\in \Y}( \ell(v, i )  - \phi_i^{*c}(\tilde x)) m_i, \quad \forall \tilde x \in \X,
\label{optimal classifier}
\end{equation} 
is a solution to problem \eqref{eqn:ATGeneralLoss}, i.e., it is an optimal robust classifier for the adversarial model \eqref{eqn:ATGeneralLoss} with $\F= \F_{\mathrm{all}}$.
\label{thm:duality}
\end{theorem}

\begin{remark}
\label{rem:Truncation}
  Just as with problem \eqref{dual_0-1} for the 0-1 loss setting, problem \eqref{dual form} is a very advantageous reformulation of \eqref{eqn:ATGeneralLoss} for computational and analytical purposes. Indeed, problem \eqref{dual form} with $\G = C_b(\X)$ reduces to a \textit{finite-dimensional} convex problem when the clean data distribution $\mu$ is an empirical measure over finitely many observations (i.e., the most important setting in applications), while the original formulation \eqref{eqn:ATGeneralLoss} does not a priori suggest this form. Moreover, problem \eqref{dual form} lends itself to natural relaxations with improved computational complexity. Indeed, a possible computational strategy, explored in \cite{ATThroughOT} and in \cite{NEURIPS2023_9b867f0e}, is to consider a truncation of class interactions in \eqref{dual form} and, for example, restrict the constraints to subsets $A$ of $\Y$ with cardinality smaller than a certain fixed (smaller than $|\Y|$) value; conveniently, truncations of this form will continue to produce valid lower bounds for the original adversarial problem \eqref{eqn:ATGeneralLoss}, even if they are not necessarily sharp. From an analytical perspective, we note that the expression \eqref{optimal classifier} for $f^*$ will typically imply some regularity estimates for optimal robust classifiers in terms of the regularity of the cost function $c$; see, for example, Remark \ref{rem:OnRegularity} below.
\end{remark}

\begin{remark}
From our proofs in section \ref{sec:ProofsMain} it is apparent that the equivalence of \eqref{dual form} with \eqref{eqn:ATGeneralLoss} holds for any family $\G$ of Borel measurable functions containing $C_b(\X)$ with the following property: for any element $\phi \in \G$, $\phi^c$ is Borel measurable. When the cost $c$ is continuous, the latter condition is automatically satisfied as in that case the $c$-transform of any Borel measurable function is upper-semicontinuous (hence Borel measurable). Likewise, when the measures $\mu_i$ are concentrated on finitely many points, $c$-transforms of Borel measurable functions are always Borel measurable.      
\end{remark}

\begin{remark}
We emphasize that the measurability of $f^*$ in Theorem \ref{thm:main} is an \textit{assumed} condition. Indeed, while Borel measurability follows from continuity of the cost function $c$, more care is needed to deduce the existence of solutions to \eqref{eqn:ATGeneralLoss} for more irregular cost functions such as the one in \eqref{eqn:0inftyCost}; we refer the interested reader to \cite{ExistenceSolutionsAT}, where some of these issues are discussed for the case of loss $0$-$1$. Regarding the uniqueness of solutions, we remark that this may depend on both the cost function $c$ as well as on the loss function $\ell$. Indeed, even when the loss function is the cross-entropy loss, which is a strictly convex function, problem \eqref{eqn:ATGeneralLoss} may not have unique solutions for cost functions like the one in \eqref{eqn:0inftyCost}. This is because, in that setting, the objective function in \eqref{eqn:ATGeneralLoss} does not penalize the values of a classifier outside the set of points that lie within distance $\veps$ from the support of the clean data distribution $\mu$. Other, more general notions like the one investigated in \cite{Frank2025} for binary classification would need to be considered in order to deduce a form of uniqueness of solutions for the problems studied in this paper.

\end{remark}

Problem \eqref{dual form} can be understood as an equivalent reformulation of \eqref{eqn:ATGeneralLoss} from the learner's perspective. On the other hand, it is possible to derive another general purpose reformulation of \eqref{eqn:ATGeneralLoss} that is analogous to the generalized barycenter problem \eqref{01_barycenter} for the 0-1 loss and that can be understood as a problem solved by the adversary. This is expressed in the following theorem, where we make additional structural assumptions on the loss function $\ell$ for interpretability.

\begin{theorem}
Suppose that the cost function $c$ satisfies Assumption \ref{assump:Cost} and let $\ell $ be a loss function satisfying Assumption \ref{assump:LossFunction}, with the additional structure
\[ \ell(v, i) = \beta(v_i), \quad  v\in \Delta_\Y, \, i \in \Y, \]
for a convex and non-increasing function $ \beta : \R_+ \rightarrow \R\cup \{ \infty\}$. Let $\varphi$ be the function defined according to
\begin{equation}
    \varphi(s) := - \inf_{t>0} \{ \beta(t) s +t  \}.
    \label{eqn:Varphi}
\end{equation}
Then \eqref{eqn:ATGeneralLoss} with $\F = \F_{\mathrm{all}}$ is equal to:
    \begin{equation}
   - \inf_{   (\tilde \mu_i)_{i \in \Y},  \lambda \in \M_+(\X)} \, \left\{  \lambda(\X)  + \sum_{i \in \Y } \int_\X \varphi \left(\frac{d \tilde \mu_i}{d\lambda}\right ) d\lambda   + \sum_{i \in \Y} C(\mu_i, \tilde \mu_i) \right\},
\label{eqn:ATGeneralLoss_dual_barycentric}
    \end{equation}
    where the inf ranges over positive finite measures $\{ \tilde{\mu}_i \}_{i \in \Y}, \lambda$ over $\X$, and where we implicitly assume $\tilde{\mu}_i \ll \lambda$ for all $i \in \Y$ (for otherwise we interpret the objective function as equal to $+\infty$). 
    \label{Thm: Dual as generalised barycenter}
\end{theorem}

Problem \eqref{Thm: Dual as generalised barycenter} is another form of generalized barycenter problem over positive measures, but with a penalty term $\int_\X \varphi \left(\frac{d \tilde \mu_i}{d\lambda}\right ) d\lambda$ that (in general) replaces the hard constraints $\tilde{\mu}_i \leq \lambda$ in \eqref{01_barycenter}. For the cross-entropy loss (see \eqref{def:CrossEntropy} for a precise definition), we can interpret the resulting problem \eqref{eqn:ATGeneralLoss_dual_barycentric} as a generalized barycenter problem with a Kullback-Leibler type penalization, which relaxes the hard constraint $\tilde{\mu}_i \leq \Lambda $ in \eqref{01_barycenter}. For certain loss functions with a power form that we will refer to as \textit{$\alpha$-logarithmic losses}, the reformulation \eqref{eqn:ATGeneralLoss_dual_barycentric} becomes a generalized barycenter problem with a suitable Tsallis relative entropy penalization (see  \ref{def:Tsallis} below). Other notions of barycenter problems for unbalanced measures have been recently explored in papers like \cite{Gero,Gramfort,Heinemann2022,JMLR:v26:22-1262}.

\subsection{Main results for some examples of loss functions}
\label{sec:Examples}
After presenting our main results in a general but somewhat abstract way, we make our results concrete by discussing more explicit forms for Theorems \ref{thm:main} and \ref{Thm: Dual as generalised barycenter} for some important examples of loss functions. We state each result in the form of a corollary and provide a brief discussion of its implications. The proofs of the results enunciated in this section are presented in section \ref{sec:ProofsConcreteLosses} below.

\subsubsection{Cross-entropy loss}
\label{Subsec: Optimal classifier cross-entropy}
Recall that the cross-entropy loss function is defined as
\begin{equation}
\ell_{\mathrm{ce}}(v,i) := - \log(v_i), \quad v \in \Delta_\Y, \, i \in \Y, 
\label{def:CrossEntropy}
\end{equation}
which clearly satisfies Assumption \ref{assump:LossFunction}. In the following corollary, we provide an explicit formula for the optimal robust classifier $f^*$ in \eqref{optimal classifier} when $\ell= \ell_{\mathrm{ce}}$.

\begin{corollary}[Form of optimal classifier for the cross-entropy loss] Provided Assumption \ref{assump:Cost} on the cost function $c$ is satisfied, if $\{\phi_i^*\}_{i \in \Y}$ is a solution to \eqref{dual form} for $\ell = \ell_{\mathrm{ce}}$, then the optimal classifier $f^*$ in \eqref{optimal classifier} can be explicitly written as 
\begin{equation}
    f_i^*(\tilde x) = \frac{ \exp \left( - \phi^{*c}_i (\tilde x)   \right)  }{ \sum_{j \in \Y}  \exp \left(- \phi^{*c}_j  (\tilde x)   \right)  }, \quad i\in \Y,
    \label{classifier cross-entropy pure}
\end{equation}
where we recall $\phi_i^{*c}$ is the $c$-transform of $\phi^*_i$ as introduced in \eqref{def c-transform}.
\label{cor:CrossEntropyOptimizer}
\end{corollary}

\begin{remark}
\label{rem:OnRegularity}
It is straightforward to show that when $c(x, \tilde x)= \frac{1}{\tau}d(x, \tilde x)$ (recall that $d$ is the distance in $\X$), each of the functions $f_i^*$ in \eqref{classifier cross-entropy pure} is $1/\tau$-Lipschitz. This is a particular instance of the fact that, typically, the $c$ transform of a function will directly inherit some regularity from the cost function $c$.
\end{remark}

We also consider the formulation \eqref{Thm: Dual as generalised barycenter} for the cross-entropy.

\begin{corollary}[Barycenter formulation for the cross-entropy loss]
Assume that the loss $\ell$ is the cross-entropy loss given in \eqref{def:CrossEntropy}. Then the generalized barycenter problem \eqref{eqn:ATGeneralLoss_dual_barycentric} is equivalent to
    \begin{equation}
        1 - \inf_{   (\tilde \mu_i)_{i \in \Y},  \lambda \in \M_+(\X)} \left\{   \lambda(\mc X) +   \sum_{i\in \Y} \mathrm{KL}(\tilde \mu_i | \lambda  )  
        + \sum_{i\in \Y } C(\mu_i, \tilde \mu_i) \right\},        \label{eq:barycentric_cross_entropy_expression}
    \end{equation}  
where $ \mathrm{KL}(\tilde \mu_i| \lambda)$ is equal to $ \int_{\X} \log \left( \frac{d \tilde \mu_i}{d \lambda} \right) d \tilde \mu_i$ if $\tilde \mu_i \ll \lambda$, and $+\infty$ otherwise. 
    \label{cor:CrossEntropyBarycenter}
\end{corollary}
Note that \eqref{eq:barycentric_cross_entropy_expression} is a generalized barycenter problem with an additional Kullback-Leibler penalization term. We remark that, since $\lambda$ and $\tilde{\mu}_i$ don't necessarily have the same total mass, $\mathrm{KL}(\tilde \mu_i|\lambda)$ as defined above may take on negative values. On the other hand, we can show that, for $\{ \tilde{\mu}_i \}_{i \in \Y}$ with finite cost $\sum_{i \in \Y} C(\mu_i, \tilde \mu_i)$, the value $\lambda(\mc X) +   \sum_{i\in \Y} \mathrm{KL}(\tilde \mu_i | \lambda  ) $ is always bounded from below by  $1+\log(1/k)$; see Remark \ref{rem:CrossEntropyLB} in the Appendix.

\medskip

While the above results hold for arbitrary cost functions $c$ satisfying Assumptions \ref{assump:Cost}, there is a further simplification of the problem \eqref{dual form} for specific choices of $c$. For example, when the cost function is chosen as $c= c_\veps$ with $c_\veps$ as in \eqref{eqn:0inftyCost}, which, as we discussed earlier, is directly related to the adversarial training model \eqref{eqn:StnadardAT}, we have the following result.

\begin{corollary}[Cross-entropy loss with $0$-$\infty$ cost]
Assume that the loss is the cross-entropy loss from \eqref{def:CrossEntropy} and suppose, in addition, that the cost function $c: \X  \times \X \rightarrow \R_+ \cup \{ \infty \}$ is the $0$-$\infty$ cost defined in \eqref{eqn:0inftyCost}. Then the value of problem \eqref{dual form} is the same as the value of:
 \begin{equation}
    \begin{aligned}
    &\inf_{ \{ \psi_i \}_{i \in \Y} \subseteq \G_0  }& \quad   -  &\sum_{i \in \Y} \int_\X  \log(\psi_i(x_i) )  d \mu_i (x_i), \\
    & \qquad \mathrm{s.t.}& & \sum_{i\in A} \psi_i(x_i) \leq   1, \quad \forall x_A \in \spt(\mu_A)\,\, \mathrm{ s.t. } \,  \bigcap_{i \in A} B_\veps(x_i) \not = \emptyset, \quad \forall A \subseteq \Y,
    \label{dual_CE}
      \end{aligned}
\end{equation}
where $\G_0$ is the set of functions of the form $\exp(\phi)$ for $\phi\in \G$. Moreover, if $\{ \psi_i^*\}_{i \in \Y}$ is a solution of \eqref{dual_CE}, then $\{ \phi_i^*:= \log(\psi_i^*)\}_{i \in \Y}$ is a solution of \eqref{dual form}. In particular, it is possible to directly obtain an optimal robust classifier $f^*$ from a solution of \eqref{dual_CE} for a $\G$ containing $C_b(\X)$. 
    \label{cor:CrossEntropyDual}
\end{corollary}

Note that when $\G$ is $ C_b(\X)$ or the set of all measurable functions, the two sets $\G_0$ and $\G$ coincide. Also, note that the expression \eqref{dual_CE} is quite similar to the one for the 0-1 loss discussed earlier in this introduction, with the difference that in the objective function of \eqref{dual_CE} it is now the logarithms of the variables $\psi_i$ and not the $\psi_i$ themselves that appear. This new expression has the form of a minimization problem with a logarithmic objective function subject to linear constraints and is a special case of a \textit{$\alpha$-fair packing problem} (with $\alpha =1)$; see Appendix \ref{app:AlphaFair}.   


\begin{remark}
    As explained before, having certain numerical methods for $\alpha$-fair packing in mind, the idea of the reformulation \eqref{dual_CE} is that the non-linearity in the problem appears in the objective function, while the constraints remain linear. Of course, with the change of variables $\phi= \log(\psi)$ it is possible to return to the setting of a linear objective with nonlinear (but convex) constraints.
\end{remark}

\subsubsection{$\alpha$-logarithmic loss}  
\label{sec:alpha-fair}

Next, we consider a family of loss functions that we will refer to as \textit{$\alpha$-logarithmic} losses. Specifically, for a given $\alpha \geq 0$ and $\alpha\not =1$, the $\alpha$-logarithmic loss is defined as
\begin{equation}
\ell_{\alpha}(v,i) := - \log_\alpha(v_i), \quad \log_\alpha(t) := \frac{t^{1-\alpha} -1}{1-\alpha}, \quad t > 0.
\label{def:AlphaLoss}
\end{equation}

\begin{remark}
Note that, regardless of the specific value of $\alpha$, the $\alpha$-logarithm function $\log_\alpha$ is both increasing and concave in its domain and thus $\ell_\alpha$ satisfies Assumption \ref{assump:LossFunction}. Also, note that, as $\alpha \rightarrow 1$, we recover the cross-entropy loss $\ell_{\mathrm{ce}}$, while we obtain the 0-1 loss \eqref{eqn:01} when we set $\alpha =0$. The family of $\alpha$-logarithmic losses \eqref{def:AlphaLoss} thus interpolates between the 0-1 and cross-entropy loss functions. In economic theory, the functions $\log_\alpha$ are known as \textit{isoelastic utilities}; see the Appendix for some discussion.
\end{remark}

\begin{remark}
\label{rem:DomainexpAlpha}
For all $\alpha \geq 0$ with $\alpha \not =1$, the function $\log_\alpha$ is continuous and strictly increasing and thus invertible over its range. In the sequel, we denote its inverse by $\exp_\alpha$ (the $\alpha$-exponential) and use $\log_\alpha$ and $\exp_\alpha$ to characterize optimal robust classifiers in the setting of the loss function $\ell_\alpha$. In particular, it will be important to precisely specify $\log_\alpha$'s range, which determines $\exp_\alpha$'s domain. To do this, we must distinguish between two separate cases that, as we will soon discuss, induce very different qualitative behaviors on the corresponding adversarial models; see in particular Remark \ref{rem:QualitativeBehavalpha} below. First, note that, in case $\alpha \in [0,1)$, the range of $\log_\alpha$ is $[-\frac{1}{1-\alpha}, \infty)$ and $\log_\alpha(0) = - \frac{1}{1-\alpha}$. On the other hand, in case $\alpha >1$ the function $\log_\alpha$ has range $(-\infty, -\frac{1}{1-\alpha})$ and $\lim_{t \rightarrow 0^{+}} \log_\alpha(t)= -\infty$. In either case, the function $\exp_\alpha$ is strictly increasing and convex and can be written as
\begin{equation}
   \exp_\alpha (s) = \left((1-\alpha)s + 1\right)^{1/(1-\alpha)},
   \label{def:ExpAlpha}
\end{equation}
provided $s$ belongs to the suitable domain. For the convenience of the reader, in Figure \ref{Fig:Plots} we visualize the function $\log_\alpha$ in the two cases $0 \leq \alpha <1$ and $\alpha >1$.
\end{remark}

\begin{figure}
    \centering
    \includegraphics[width=\linewidth, trim={0 5cm 0 5cm}]{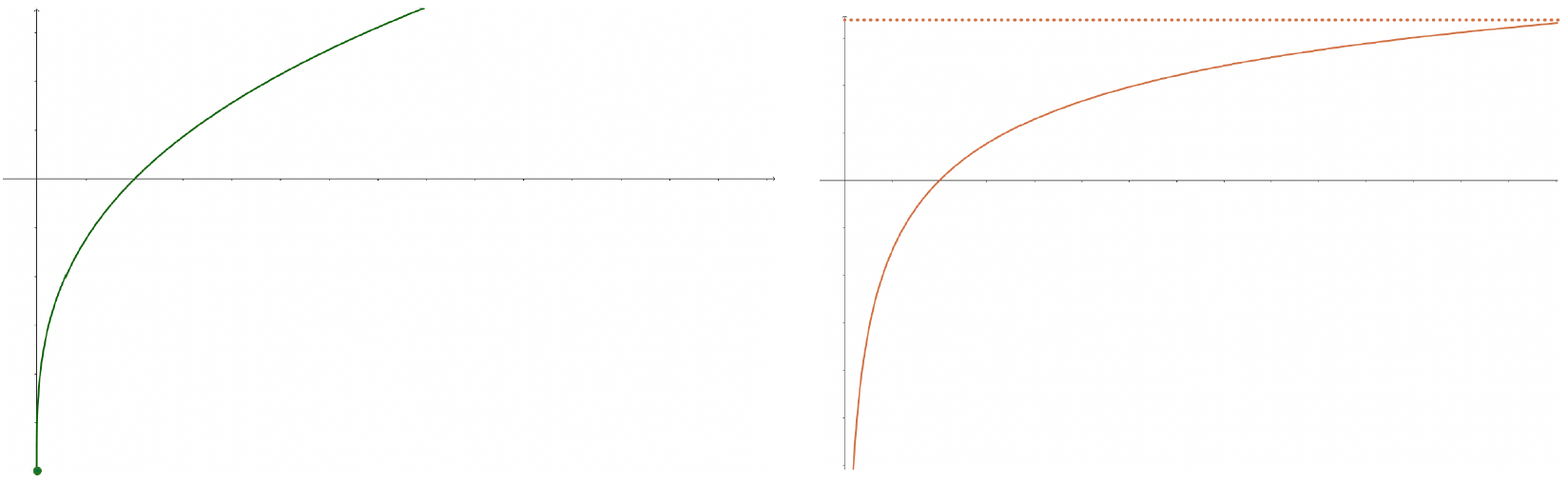}
    \put(-466, 22){{\ggreen $-\frac{1}{1-\alpha}$}}
    \put(-237, 150){{\oorange$-\frac{1}{1-\alpha}$}}
    \caption{\textit{Left:} Plot of $\log_\alpha$ when $\alpha \in [0,1)$. The function cuts the vertical axis at the value $-\frac{1}{1-\alpha}$ and diverges to $\infty$ as the argument of the function gets larger. \textit{Right:} Plot of $\log_{\alpha}$ for $\alpha>1$. The function has a horizontal asymptote at $-\frac{1}{1-\alpha}$ and a vertical one at $0$. For both cases, and regardless of the value of $\alpha$, the function $\log_\alpha$ cuts the horizontal axis at the value $1$.}
    \label{Fig:Plots}
\end{figure}

\begin{corollary}[Form of optimal classifier for $\alpha$-logarithmic loss] Let $\alpha \geq 0$, $\alpha \not =1$. Provided Assumption \ref{assump:Cost} on the cost function $c$ is satisfied, if $\{\phi_i^*\}_{i \in \Y}$ is a solution to \eqref{dual form} for $\ell = \ell_{\alpha}$, then the optimal classifier $f^*$ in \eqref{optimal classifier} can be explicitly written as 
\begin{equation}
    f_i^*(\tilde x) =  \exp_\alpha \left( \max \left\{ - \phi_i^{*c}(\tilde x) - Z(\tilde x), -\frac{1}{1-\alpha}  \right\} \right), \quad i\in \Y,
    \label{classifier alpha<1}
\end{equation}
in case $\alpha \in [0,1)$, and as
\begin{equation}
    f_i^*(\tilde x) =  \exp_\alpha \left(  - \phi_i^{*c}(\tilde x) - Z(\tilde x) \right), \quad i\in \Y,
    \label{classifier alpha>1}
\end{equation} 
when $\alpha >1$. In either case, $Z(\tilde x)$ is a ``normalization'' factor that guarantees that 
\[  \sum_{i\in \Y} f_i^*(\tilde x) =1. \]
We recall $\phi_i^{*c}$ is the $c$-transform of $\phi^*_i$ as introduced in \eqref{def c-transform}.
\label{cor:OptimalClassAlpha}
\end{corollary}

\begin{remark}
\label{rem:Porpertiesexp+Alpha}
We highlight that the normalization factor $Z(\tilde x)$ in \eqref{classifier alpha<1} and \eqref{classifier alpha>1} can always be found from the values of $\phi_i^{*c}(\tilde x)$. To see this, let $\{ a_i \}_{i \in \Y}$ be a collection of real numbers (playing the role of the $\phi_i^{*c}(\tilde x)$). For $\alpha \in [0,1)$, we observe that the function
\[ Z \in \R \mapsto \sum_{i \in \Y}  \exp_\alpha \left( \max \left\{ - a_i - Z, -\frac{1}{1-\alpha}  \right\} \right)  \]
is continuous, decreasing, and has limit $0$ when $Z \rightarrow \infty$, and $\infty$ when $Z \rightarrow -\infty$. The intermediate value theorem implies that it is always possible to find $Z$ at which this function takes the value $1$. The uniqueness of this $Z$ follows from the fact that $\exp_\alpha$ is strictly increasing. 

Likewise, for $\alpha>1$, we observe that the function
\[ Z \in ( \frac{1}{1-\alpha}  -\min_{i \in \Y} a_i , \infty) \mapsto  \sum_{i \in \Y}  \exp_\alpha \left(  - a_i - Z \right)  \]
is decreasing and continuous, and has limit $0$ when $Z \rightarrow \infty$, and $\infty$ when $Z \rightarrow \frac{1}{1-\alpha}  -\min_{i \in \Y} a_i$. It is thus possible to find $Z$ at which this function takes the value $1$. The uniqueness of this $Z$ follows, again, from the fact that $\exp_\alpha$ is strictly increasing.
\end{remark}

\begin{remark}
In case $\alpha=0$ (i.e., when $\ell_\alpha = \ell_{\mathrm{01}}$), we have $\exp_\alpha(s)=1+s$, and the optimal robust classifier can be written as
\[  f_i^*(\tilde x) = 1+ \max\{ - \phi_i^c(\tilde x) - Z(\tilde x) , -1   \} = \max \{ - \phi_i^c(\tilde x) + 1 - Z(\tilde x), 0   \}.\]
This expression has a similar form to \eqref{eq:RobustClassif0-1} (derived in \cite{MOTJakwang} and discussed further in \cite{ExistenceSolutionsAT}) after we consider the change of variables $g_i= 1+\phi_i$. The apparent discrepancy in the two formulas is resolved after noticing that in \cite{ExistenceSolutionsAT} the $f_i^*$ are only assumed to sum to a number smaller than one (which can be directly related to the problem considered here); see Remark 2.2 in \cite{ExistenceSolutionsAT}. 
\end{remark}

\begin{remark}
\label{rem:QualitativeBehavalpha}
There is a fundamentally different qualitative behavior between the robust classifiers arising from the $\alpha$-logarithmic loss model when $\alpha \in [0,1)$ and when $\alpha >1$. Indeed, when $\alpha \in [0,1)$, $f^*_i$ in \eqref{classifier alpha<1} may take the value $f_i^*=0$, whereas $f_i^*$ is guaranteed to be strictly greater than zero when $\alpha>1$. This is because the loss function $\ell_\alpha$ blows up at values close to zero in the latter case while it converges to a finite value in the former, according to the discussion in Remark \ref{rem:DomainexpAlpha}. In this sense, the loss functions $\ell_\alpha$ for $\alpha>1$ behave like the cross-entropy loss, while $\ell_\alpha$ for $\alpha \in [0,1)$ induces \textit{sparsity} and behaves more similarly to the 0-1 loss.
\end{remark}

Next, we specialize Theorem \ref{Thm: Dual as generalised barycenter} to the case of the $\alpha$-logarithmic loss function $\ell_\alpha$.

\begin{corollary}[Barycenter formulation for the $\alpha$-logarithmic loss]
Assume that the loss is given by equation \eqref{def:AlphaLoss} for some $\alpha\geq 0$ different from one. Then the generalized barycenter problem \eqref{eqn:ATGeneralLoss_dual_barycentric} is equivalent to
    \begin{equation}
        1 - \inf_{   (\tilde \mu_i)_{i \in \Y},  \lambda \in \M_+(\X)} \left\{  \lambda(\mc X) 
        + \sum_{i \in \Y}  \mathrm{D}_q(\tilde \mu_i | \lambda  )  + \sum_{i\in \Y} C(\mu_i, \tilde \mu_i)    \right\}, 
        \label{eq:barycentric_alpha_fair_expression}
    \end{equation}  
    where $q= \frac{1}{\alpha}$, and $\mathrm{D}_q(\tilde \mu_i | \lambda)$ is the $q$-Tsallis relative entropy between $\tilde \mu_i$ and $\lambda$:
\begin{equation}
\mathrm{D}_q(\tilde \mu_i| \lambda) := \int_{\X} \frac{\left(\frac{d \tilde \mu_i}{d \lambda} \right)^{q-1} - 1 }{q-1}   d\tilde \mu_i, 
\label{def:Tsallis}
\end{equation}
if $\tilde \mu_i \ll \Lambda$, and $+\infty$ otherwise. In case $\alpha=0$, i.e., when $q = \infty$, the above must be interpreted as $0$ if $\tilde \mu_i \leq \lambda$ and $\infty$ otherwise. 
    \label{cor:alphaFairBarycenter}
\end{corollary}

Similarly to the cross-entropy case, $\mathrm{D}_q(\tilde \mu_i|\lambda)$ as defined above may take on negative values. On the other hand, we can show that, for $\{ \tilde{\mu}_i \}_{i \in \Y}$ with finite cost $\sum_{i \in \Y} C(\mu_i, \tilde \mu_i)$, the quantity $\lambda(\mc X) +   \sum_{i\in \Y} \mathrm{D}_q(\tilde \mu_i | \lambda  ) $ is always bounded from below by  $1+ \log_\alpha(1/k)$; see Remark \ref{rem:CrossEntropyLB} in the Appendix.

\medskip

As for the case of the cross-entropy, when the cost function $c$ is of the form \eqref{eqn:0inftyCost} we can rewrite problem \eqref{dual form} for the $\alpha$-logarithmic loss in the following equivalent form.

\begin{corollary}[$\alpha$-logarithmic loss with $0$-$\infty$ cost]
Assume that the loss is the $\alpha$-logarithmic loss from \eqref{def:AlphaLoss} for some $\alpha \geq 0$ with $\alpha \not =1$. Suppose, in addition, that the cost function $c: \X  \times \X \rightarrow  \R_+ \cup \{\infty\}$ is the $0$-$\infty$ cost defined in \eqref{eqn:0inftyCost}. Finally, let $\G$ be any set of measurable functions on $\X$ in case $\alpha>1$, and let $\G$ be a set of measurable functions that is closed under pointwise maximum with a constant (i.e., if $\phi \in \G$, then $\max \{\phi, a \} \in \G$ for any $a \in \R$) in case $\alpha \in [0,1)$. Then the value of problem \eqref{dual form} is the same as the value of
 \begin{equation}
    \begin{aligned}        
    &  \inf_{ \{ \psi_i \}_{i \in \Y} \subseteq \G_0 }& \    -  &\sum_{i \in \Y} \int_\X \log_\alpha(\psi_i(x_i))  d \mu_i (x_i) , \\
    & \quad \mathrm{s.t.}& & \sum_{i\in A} \psi_i(x_i) \leq 1, \, 0 \leq \psi_i(x_i),\, \forall x_A \in \spt(\mu_A) \,  \mathrm{ s.t. }\,  \bigcap_{i \in A} B_\veps(x_i) \not = \emptyset, \,  \forall A \subseteq \Y,
    \label{dual_alpha fair}
      \end{aligned}
\end{equation}
where $\G_0$ is the set of functions of the form $\exp_\alpha(\phi)$ for $\phi\in \G$. Furthermore, if $\{ \psi_i^* \}_{i \in \Y}$ is a solution of the above problem, then $\{ \phi_i^*:= \log_\alpha(\psi_i^*)\}_{i \in \Y}$ is a solution of \eqref{dual form}. In particular, it is possible to directly obtain an optimal robust classifier $f^*$ from a solution of \eqref{dual_alpha fair} when $\G$ contains $C_b(\X)$.
\label{cor:AlphaFairDual}
\end{corollary}

As for the cross-entropy case, when $\G$ is $C_b(\X)$ or the set of all measurable functions we have $\G_0=\G$. Also, note that problem \eqref{dual_alpha fair} is quite similar to \eqref{dual_CE}, where, instead of having a standard logarithm in the objective, we use an $\alpha$-logarithm (a power function). As for the cross-entropy case, when $\mu$ is an empirical measure problem \eqref{dual_alpha fair} can be solved using algorithms designed for the $\alpha$-fair packing problem such as those presented in \cite{Diakonikolas_Fair_2020}.

\begin{remark}
Since the 0-1 loss $\ell_{01}$ is precisely $\ell_\alpha$ with $\alpha=0$, it is natural to ask whether the results obtained in this paper recover the results in \cite{MOTJakwang} and \cite{ATThroughOT} that were discussed earlier in this introduction. In Appendix \ref{app:Linear}, we show that this is indeed the case.     
\end{remark}

\nc

\subsubsection{Quadratic loss}
\label{Subsec: Optimal classifier quadratic}
The last example considered in this paper is the quadratic loss function defined according to 
\begin{equation}
\ell_Q(v,i):= \lVert v - e_i \rVert^2, \quad v \in \Delta_\Y,
\label{eq:QuadraticLoss}
\end{equation}
where we identify the label set $\Y$ with the set $\{1, \dots, K \}$ (for $K= |\Y|$) and $e_1, \dots, e_K$ is the canonical basis for $\R^K$, for convenience. In contrast to the previous examples, $\ell(v, i)$ depends on all entries of $v$ and not just on the $i$-th entry of $v$. For concreteness, we only present the form of the optimal classifier \eqref{optimal classifier} in this case.

\begin{corollary}[Form of optimal classifier for quadratic loss] Provided Assumption \ref{assump:Cost} on the cost function $c$ is satisfied, if $\{ \phi_i^*\}_{i \in \Y}$ is a solution to \eqref{thm:main} for $\ell = \ell_{Q}$, then the optimal classifier $f^*$ in \eqref{optimal classifier} can be written as follows for a given $\tilde x$: after relabeling the indices $i \in \Y $ so that $\phi_1^{*c}(\tilde x) \leq \dots \leq \phi_K^{*c}(\tilde x)$, and defining $i^*$ and $c^*$ according to 
\[ i^* := K \wedge \min \{ i = 1, \ldots, K \quad  \mathrm{ s.t. } \quad   i \phi^{*c}_{i+1}(\tilde x) -   \sum_{j=1}^{i}\phi_j^{*c}(\tilde x) >2    \}\] 
\[c^* := \frac{1}{i^*} (2+ \sum_{i=1}^{i^*} \phi^{*c}_i(\tilde x)), \]
we have
\begin{equation}
f_i^*(\tilde x) := 
\begin{cases}
    \frac{1}{2} (c^*  - \phi^{*c}_i(\tilde x)),  & \mathrm{ if }\, i\leq i^*, \\
    0, & \mathrm{else}.
\end{cases}   
    \label{classifier quadratic}
\end{equation}

\label{cor:OptimalClassQuadratic}
\end{corollary}

\subsection{Related literature} 
There is a growing literature on lower bounds in adversarial classification. Architecture-specific results were obtained in \cite{yin2019rademacher} for linear classifiers; see also \cite{khim2018adversarial} for linear classifiers and neural networks. An alternative approach, based on obtaining classifier-agnostic bounds that hold regardless of model architectures, which is the perspective explored in this paper, was pioneered by \cite{bhagoji2019lower} using optimal transport theory in the setting of binary classification and 0-1 loss function. This analysis was later extended by \cite{VarunMuni2}, where more detailed existence and characterization results were provided; other related results have been established in \cite{awasthi2021existence, frank2023the, frank2023existence}. Still in the binary classification setting, the work \cite{Carlier} proves the existence of continuous optimal robust classifiers assuming sufficient regularity of the loss function. In \cite{LeonNGTRyan}, adversarial training was studied through a geometric perspective, and in \cite{Bungert_Gammaconvergence_2024} the authors studied a related notion of nonlocal perimeter and used $\Gamma$-convergence techniques to study the behavior of solutions to adversarial training in the small adversarial budget regime. A related paper is \cite{RyanRachel}, where a stronger form of convergence characterizing the asymptotic behavior of robust classifiers in the small adversarial budget regime was considered. We also mention the work \cite{BUNGERT2024103625}, which provides a deeper connection between adversarial training and geometric variational analysis.

For multiclass problems, \cite{ExistenceSolutionsAT} proved existence of solutions to the learner-agnostic adversarial risk minimization problem, \cite{ATThroughOT} exploited a connection to multimarginal optimal transport to deduce computationally tractable algorithms to compute lower bounds, while \cite{NEURIPS2023_9b867f0e} characterized the problem through the notion of conflict hypergraph. All these results consider the 0-1 loss.

\nc

\subsection{Outline} The rest of the paper is organized as follows. 
Section \ref{sec:Proofs} contains most of the proofs of the theoretical results stated in sections \ref{sec:IntroMianResults} and \ref{sec:Examples}. We begin by proving Theorem \ref{thm:main}, and, in the process, we develop other equivalent reformulations of the original problem \eqref{eqn:ATGeneralLoss} that could potentially be used to design alternative methods to solve it. In section \ref{sec:ProofsConcreteLosses}, we present the proofs of the results for the cross-entropy, $\alpha$-logarithmic, and quadratic loss functions that we stated in section \ref{sec:Examples}. In section \ref{subsec:Lowerbounds}, we use our theoretical results to derive lower bounds for the robust training of learning models in a simple, yet concrete practical setting.

In the Appendix, we present additional technical auxiliary results used in the proof of Theorem \ref{thm:main}, present the proof of Theorem \ref{Thm: Dual as generalised barycenter}, present a brief discussion of $\alpha$-fair packing, and provide more details on how the results derived in this paper recover the reformulations of \eqref{eqn:ATGeneralLoss} for the 0-1 loss case derived in \cite{MOTJakwang} and \cite{ATThroughOT}.

\medskip 

\textit{Additional notation:} We use $\M_+(\X)$ and $\M_+(\X \times \X)$ to denote the set of finite positive measures over $\X$ and $\X \times \X$, respectively. For a given $\nu \in \M_+(\X)$, we use $\spt(\nu)$ to denote $\nu$'s support and use $\Gamma_1(\nu)$ to denote the set of measures $\pi \in \M_+(\spt(\nu) \times \X)$ whose first marginal is equal to $\nu$. In the sequel, we may use $\X_i$ to represent the set $\spt(\mu_i)$, especially when notation gets particularly burdensome.

Given $\nu, \tilde \nu \in \M_+(\X)$, we use $\Gamma(\nu, \tilde \nu)$ to represent the set of couplings between $\nu$ and $\tilde \nu$, i.e., the set of $\pi \in \M_+(\X \times \X)$ whose first and second marginals are $\nu$ and $\tilde \nu$, respectively. Note that the set $\Gamma(\nu, \tilde \nu)$ is nonempty if and only if $\nu(X)= \tilde{\nu}(X)$.

For a given $x \in \X$, we use $B_\veps(x)$ to denote the closed ball of radius $\veps$ around $x$. In the sequel, we may identify $\Y$ with the set $\{1, \dots, K \}$ (where $K= |\Y|$) without further mention. We use $u \odot v$ to denote the Hadamard product (coordinatewise product) between two vectors of the same dimension. The vectors $\vec{0}_K, \vec{1}_K$ are the $K$-dimensional vectors with all zeroes and all ones, respectively. $\mathbb{I}_K$ represents the $K\times K$ identity matrix. The symbol $\otimes $ is used to describe product measures (as in $\otimes _{i \in A} \mu_i$) or tensor products between two vectors (as in $u \otimes v$); in the latter case, $u \otimes v$ is the matrix whose $ij$ entry is $u_iv_j$. No confusion should arise about the intended use of the symbol $\otimes$.

\section{Proofs}
\label{sec:Proofs}
In this section, we present the proofs of the results stated in the introduction, with the exception of the proof of Theorem \ref{Thm: Dual as generalised barycenter}, which is postponed to the Appendix. We begin with the proof of Theorem \ref{thm:main}.

\subsection{Proof of Theorem \ref{thm:main}}
\label{sec:ProofsMain}
We first derive some useful inequalities using weak duality arguments.

\begin{proposition}
Suppose that $\ell(\cdot, i): \Delta_\Y \rightarrow [0,\infty] $ is a continuous function for all $i \in \Y$. Then the value of problem \eqref{eqn:ATGeneralLoss} for $\F=\F_{\mathrm{all}}$ is smaller than or equal to the value of
    \begin{equation}
    \begin{aligned}
    &\inf_{(\phi_i)_{i \in \mc Y} \in \G , f \in \F_{\mathrm{all}}   }& \quad  -  &\sum_{i\in \mc Y} \int_\X \phi_i(x) d\mu_i(x), \\
    & \qquad \mathrm{s.t.}& & -\phi_i(x) + c(x, \tilde x) \geq \ell(f(\tilde x),i),  \quad \forall x \in \spt(\mu_i),\tilde x \in \X, i\in \Y,
    \label{eq:ATGeneralLoss_Dual_0}
      \end{aligned}
      \end{equation}
      \label{Prop:First}
\end{proposition}
provided $\G$ is a set of measurable functions containing $C_b(\X)$.
\begin{proof}
Let $f$ be an arbitrary measurable soft classifier. Observe that 
\begin{align*}
  \sup _{ \{\tilde{\mu}_i\}_{i \in \mc Y}  } & \sum_{i \in \Y}   \int_\X \ell(f(\tilde x), i) d\tilde \mu_i(\tilde x)   - \sum_{i \in \Y} C(\mu_i, \tilde \mu_i) 
  \\ & =  \sup_{\pi_i \in \Gamma_1(\mu_i), \: i \in \Y} \sum_{i \in \Y} \int_{\spt(\mu_i) \times \X} (\ell(f(\tilde x), i) - c(x, \tilde x)) d\pi_i(x, \tilde x), 
   \label{eq:Aux1}
\end{align*}
where, recall, $\Gamma_1(\mu_i)$ denotes the collection of $\pi$ in $\M_+(\spt(\mu_i) \times \X)$ whose first marginal is equal to $\mu_i$.

Now, for a $\pi_i \in \M_+(\spt(\mu_i) \times \X)$, the negative of the characteristic function for the constraint $\pi_i \in \Gamma_1(\mu_i)$ can be written as
\[ \inf_{\phi_i \in \G}  \int_{\spt(\mu_i) \times \X} \phi_i(x)  d\pi_i(x, \tilde x) -   \int_\X \phi_i(x)  d\mu_i(x),\]
given that $C_b(\X) \subseteq \G$ (and $C_b(\X)$ characterizes Borel measures). From the above, we deduce that
\[ \sup_{\pi_i \in \Gamma_1(\mu_i), \: i \in \Y} \int_{\spt(\mu_i) \times \X} (\ell(f(\tilde x), i) - c(x, \tilde x)) d\pi_i(x, \tilde x)  \]
is equivalent to 
\[ \sup_{\pi_i \in \M_+(\spt(\mu_i) \times \X)} \inf_{(\phi_i)_{i \in \Y} \in \G}   - \int_\X \phi_i(x) d\mu_i(x) + \int_{\spt(\mu_i) \times \X} (\ell(f(\tilde x), i) - c(x, \tilde x) +\phi_i(x) )d\pi_i(x, \tilde x).  \]
Swapping the sup and the inf, we get the upper bound
\[ \inf_{(\phi_i)_{i \in \Y} \in \G} \sup_{\pi_i \in \M_+(\spt(\mu_i) \times \X)}    - \int_\X \phi_i(x) d\mu_i(x) + \int_{\spt(\mu_i) \times \X} (\ell(f(\tilde x), i) - c(x, \tilde x) +\phi_i(x) )d\pi_i(x, \tilde x).  \]
In turn, for fixed $\phi_i$ the inner sup over $\pi_i \in \M_+(\spt(\mu_i) \times \X)$ gives $0$ if the constraint $\ell(f (\tilde x), i) - c(x, \tilde x) + \phi_i(x)\leq 0$ is satisfied for all $\tilde x \in \X$ and all $x \in \spt(\mu_i)$, and is equal to $\infty$ if not. Inequality $\eqref{eqn:ATGeneralLoss} \leq \eqref{eq:ATGeneralLoss_Dual_0}$ follows. 
\end{proof}

\begin{proposition}
Under Assumption \ref{assump:LossFunction} on the loss function $\ell$, problem 
\begin{equation} \inf_{(\phi_1, \dots, \phi_K) \in \mathfrak{A}} \quad   -  \sum_{i\in \mc Y} \int_\X \phi_i(x) d\mu_i(x),
\label{eq:ATGeneralLoss_Dual2}
\end{equation}
for $\mathfrak{A}$ the admissible set 
\[ \mathfrak{A} := \left\{ (\phi_1, \dots, \phi_K) \in \G^K\, \mathrm{ s.t. }\, 0 \geq \min_{v \in \Delta_\Y} \max_{m \in \Delta_\Y} \sum_{i\in \Y} (\ell(v,i) - \phi_i^c(\tilde x)) m_i  , \quad \forall \tilde x \in \X \right\}, \]
is equivalent to \eqref{eq:ATGeneralLoss_Dual_0}, provided $\G = C_b(\X)$.

\label{proposition:duality}
\end{proposition}

\begin{proof}
Suppose that $\{ \phi_i \}_{i \in \Y}$ and $f$ form a feasible tuple for \eqref{eq:ATGeneralLoss_Dual_0}. Then, for all $\tilde x \in \X$ and all $i \in \Y$,
\[ \phi_i^c(\tilde x) = \inf_{x \in \text{spt}(\mu_i)} \{ c(x, \tilde x) -  \phi_i(x) \} \geq \ell(f(\tilde x), i).\]
Hence
$$0 \geq \ell(f(\tilde x), i) - \phi_i^c(\tilde x), \quad \forall \tilde x \in \X \quad \forall i \in \Y.$$
It follows that 
$$0 \geq \max_{m \in \Delta_\Y} \sum_{i \in \Y}(  \ell(f(\tilde x), i) - \phi_i^c(\tilde x) ) m_i  \geq \min_{v \in \Delta_\Y} \max_{m \in \Delta_\Y} \sum_{i \in \Y}(  \ell(v, i) - \phi_i^c(\tilde x) ) m_i, \quad \forall \tilde x \in \X,$$
 from where we conclude that $(\phi_1, \ldots, \phi_K) \in \mathfrak{A}$. In particular, $\eqref{eq:ATGeneralLoss_Dual_0} \geq \eqref{eq:ATGeneralLoss_Dual2}$. Note that this part of the argument holds for any $\G$ containing $C_b(\X)$.

Conversely, let  $(\phi_1, \ldots, \phi_K) \in \mathfrak{A}$ for $\G= C_b(\X)$. Since each $\phi_i$ is continuous and bounded, it follows from Lemma \ref{lem:MeasurabilityCTRansform} in the Appendix (which relies on Assumption \ref{assump:Cost} on the cost function $c$) that $\phi_i^c$ is Borel measurable for every $i \in \Y$. Now, by definition, for any $\tilde x \in \mc X$ we have
$$0 \geq \min_{v \in \Delta_\Y} \max_{m \in \Delta_\Y} \sum_{i \in \Y} (  \ell(v, i) - \phi_i^c(\tilde x) ) m_i.$$
Our goal is to construct a measurable function $f : \X \mapsto \Delta_\Y$ such that
\[ f(\tilde x) \in  \arg \min_{v \in \Delta_\Y}  \max_{m \in \Delta_\Y} \sum_{i \in \Y}( \ell(v, i )  - \phi_i^c(\tilde x)) m_i, \quad \forall \tilde x \in \X. \]
To do this, first consider the set-valued map
\begin{equation}
\Xi : (b_1, \dots, b_K) \in \R^{K}\longmapsto \arg\min_{v \in \Delta_\Y} \max_{m \in \Delta_\Y} \sum_{i \in \Y}(\ell(v,i) - b_i) m_i.
\end{equation}
We can verify that $\Xi$ satisfies the assumptions in the Kuratowski-Ryll-Nardzewski measurable selection theorem and thus admits a measurable selection $\xi: \R^{K} \mapsto \Delta_\Y $. The desired measurable map $f $ can then be defined as $f(\tilde x) := \xi \circ \vec{b}(\tilde x)$, where $\vec{b}(\tilde x):= ( \phi_1^c(\tilde x), \dots, \phi_K^c(\tilde x))$. For this function, which is Borel measurable given that it is the composition of two Borel measurable maps, we have
$$0 \geq \max_{m \in \Delta_\Y} \sum_{j\in \Y}(  \ell(f(\tilde x), j) - \phi_j^c(\tilde x) ) m_j \geq  \ell(f(\tilde x), i) - \phi_i^c(\tilde x), \quad \forall i \in \Y.$$
Using the definition of $\phi_i^c$ and then reordering some terms, we obtain
$$c(x_i,\tilde x) - \phi_i(x_i) \geq \phi_i^c(\tilde x) \geq \ell(f(\tilde x),i),$$
for all $i \in \Y$, $\tilde x \in \X$, $x_i \in \spt(\mu_i)$. 
We conclude that the tuple $(\phi_1, \ldots, \phi_K), f$ is feasible for \eqref{eq:ATGeneralLoss_Dual_0}. This implies the reverse inequality $\eqref{eq:ATGeneralLoss_Dual2} \geq  \eqref{eq:ATGeneralLoss_Dual_0} $. 
\end{proof}

\begin{proposition}
For any set $\G$ of Borel measurable functions on $\X$, problems \eqref{eq:ATGeneralLoss_Dual2} and \eqref{dual form} are equivalent.
 \label{prop:ThirdAux}   
\end{proposition}

\begin{proof}
It suffices to show that the condition 
\begin{equation}
0 \geq \sup_{m_A \in \Delta_{A}} \left\{ \sum_{i \in A} m_i\phi_i(x_i)  +  \ell_A(m_A) - c_A(x_A , m_A)      \right\}, \quad \forall x_A  \in \spt(\mu_A), \quad \forall A \subseteq \Y,
\label{eq:ConstraintFiniteSupport}
\end{equation}
is equivalent to $(\phi_1, \dots, \phi_K) \in \mathfrak{A}$. 

To see this, let us first assume that $(\phi_1, \dots, \phi_K) \in \mathfrak{A}$. Then for any given $\tilde x  \in \X$ there is $v \in \Delta_\Y$ such that
\[ 0 \geq \sum_{i \in \Y} (\ell(v,i) - \phi_i^c(\tilde x)  ) m_i, \quad \forall m\in \Delta_\Y. \]
By definition of $\phi_i^c(\tilde x)$ we have
\[ \phi_i(x_i) + \phi_i^c(\tilde x) \leq c(x_i, \tilde x), \quad \forall x_i \in \spt(\mu_i), \]
and thus also
\[ 0 \geq \sum_{i \in \Y} (\ell(v,i) + \phi_i(x_i) - c(x_i, \tilde x) ) m_i, \quad \forall m \in \Delta_\Y, \quad \forall x_i \in \spt(\mu_i), \quad \forall i \in \Y.  \]
If $m$ is chosen to belong to $\Delta_{A}$ for some $A \subseteq \Y$, the above implies
\[0 \geq \ell_A(m_A) + \sum_{i \in A} m_i \phi_i(x_i) - \sum_{i \in A} c(x_i, \tilde x) m_i, \quad \forall x_A \in \spt(\mu_A).\]
In particular, if for a fixed $x_A \in \spt(\mu_A)$ we take the supremum of the right hand side of the above expression over $\tilde x$, we deduce
\[ 0 \geq   \sum_{i \in A} m_i  \phi_i(x_i)  + \ell_A(m_A)  - c_A(x_A, m_A ), \quad \forall m_A \in \Delta_A. \]
Condition \eqref{eq:ConstraintFiniteSupport} follows. 

Conversely, suppose that $\{ \phi_i \}_{i \in \Y}$ satisfies \eqref{eq:ConstraintFiniteSupport}. Fix $\tilde x \in \X$ and consider $\{ x_i\}_{i \in \Y}$ with $x_i \in \spt(\mu_i)$, $i \in \Y$.
We use \eqref{eq:ConstraintFiniteSupport} with the tuple $x_\Y:= \{ x_i\}_{i \in \Y}$ to obtain
\begin{align}
\label{eq:Swap}
\begin{split}
   0 & \geq \max_{m \in \Delta_\Y} \min_{v \in \Delta_\Y} \{  \sum_{i \in \Y} m_i(  \phi_i(x_i) + \ell(v,i) )  - c_\Y(x_\Y, m)   \} 
   \\& \geq \max_{m \in \Delta_\Y} \min_{v \in \Delta_\Y} \{  \sum_{i \in \Y} m_i(  \phi_i(x_i) + \ell(v,i) )  - \sum_{i \in \Y}  c(x_i, \tilde x) m_i   \} 
   \\& =  \min_{v \in \Delta_\Y } \max_{m \in \Delta_\Y}  \{  \sum_{i \in \Y} m_i( \phi_i(x_i) + \ell(v,i) )  -  \sum_{i \in \Y} c(x_i, \tilde x) m_i   \}.
   \end{split}
\end{align}
The second inequality follows from the definition of $c_\Y(x_\Y , m)$. In the third line, we can swap the min and the max thanks to Assumption \ref{assump:LossFunction} (which implies convexity in the $v$ variable) and the linearity (in particular concavity) in the $m$ variable. It follows that for every $\tilde x$ and every tuple $\{x_i\}_{i \in \Y}$ there is $v \in \Delta_\Y$ such that
\[  0 \geq \sum_{i \in \Y} m_i(   \ell(v,i) + \phi_i(x_i) - c(x_i, \tilde x)  ), \quad \forall m \in \Delta_\Y .\]
Now, since $x_i \in \spt(\mu_i)$, $i \in \Y$, were arbitrary, we can conclude, using the definition of $\phi_i^c(\tilde x )$ and compactness of $\Delta_\Y$, that
\[ 0 \geq  \sum_{i \in \Y} m_i(\ell(v,i)- \phi_i^c(\tilde x)), \quad \forall m \in \Delta_\Y, \]
for some $v \in \Delta_\Y$. In turn, we deduce
\[ 0 \geq \min_{v \in \Delta_\Y}\max_{m \in \Delta_\Y} \sum_{i \in \Y} m_i(\ell(v,i)- \phi_i^c(\tilde x)).\]
\end{proof}

In order to close the duality gap between \eqref{dual form} for $\G = C_b(\X)$ and \eqref{eqn:ATGeneralLoss} for $\F = \F_{\mathrm{all}}$, we use the next proposition that resembles the duality theorem for multimarginal optimal transport (MMOT) but whose proof, which we present in Appendix \ref{app:MoreDetails}, requires new constructions and ideas. To enunciate it, we first introduce some notation that we use later on. 

Given $\pi \in \M_{+}(\X_1 \times \dots \times \X_K \times \Delta_\Y)$, where, recall, $\X_i = \spt(\mu_i)$, we define $P_i\pi \in \M_+(\X_i ) $  according to
\begin{equation}
\int_{\X_i} h(x_i) d P_i \pi(x_i) = \int_{\X_1 \times \dots \times \X_K \times \Delta_\Y} m_i h(x_i) d\pi(\vec{x}, m), \quad \forall h \in C_b(\X_i). 
\label{def:PiPro}
\end{equation}
Here and in the sequel, we use $\vec{x}$ as shorthand notation to represent an arbitrary tuple $(x_1, \dots, x_K)$.

\begin{proposition}
\label{prop:MMOTVersion}
Under Assumption \ref{assump:LossFunction} on $\ell$ and Assumption \ref{assump:Cost} on $c$, the value of 
\begin{equation}
   -\min_{\pi \in \mathfrak{G} } \int_{\X_1 \times \dots \times \X_K \times \Delta_\Y} (c_\Y(\vec x, m)  - \mathbf{\ell}_\Y(m)) d\pi(\vec x, m), 
   \label{eqn:MMOTVersion}
\end{equation}
where 
\begin{equation}
    \mathfrak{G}:= \{ \pi \in \M_+(\X_1 \times \dots \times \X_K \times \Delta_\Y)\quad  \mathrm{ s.t. }  \quad   P_{i }\pi = \mu_i, \quad \forall i \in \Y  \},
    \label{eqn:CouplingsWeird}
\end{equation}
is the same as the value of problem \eqref{dual form} with $\G = C_b(\X)$. We recall that $P_i \pi$ was defined in \eqref{def:PiPro}.
\end{proposition}

We are ready to prove Theorem \ref{thm:main}.

\begin{proof}[ Proof of Theorem \ref{thm:main}]
In view of Propositions \ref{Prop:First}, \ref{proposition:duality}, and \ref{prop:ThirdAux} it will be sufficient to prove that
\begin{equation}
 \sup_{ (\pi_i)_{i \in \Y} \textrm{ s.t. } \pi_i \in \Gamma_1(\mu_i) } \inf_{f \in \F_{\mathrm{all}}}   \sum_{i \in \Y}  \int_{\X_i \times \X} \ell(f(\tilde x),i) d\pi_i(x_i, \tilde x) - \sum_{i \in \Y} \int_{\X_i \times \X} c(x_i, \tilde x) d\pi_i(x_i, \tilde x)  \geq  \eqref{eqn:MMOTVersion}.
 \label{eq:AuxMainProof}
\end{equation}
Indeed, assuming the above inequality holds, we can deduce
\[ \eqref{eqn:ATGeneralLoss} \geq \textrm{LHS of } \eqref{eq:AuxMainProof} \geq \eqref{eqn:MMOTVersion} = \eqref{dual form} \geq \eqref{eqn:ATGeneralLoss},\] 
which in turn implies that the above quantities are all equal. We thus focus on establishing \eqref{eq:AuxMainProof}.

Let $\pi \in \mathfrak{G}$, and define $\pi _i \in \M_+(\X_i \times \X)$ according to
\[ \int_{\X_i \times \X} h (x_i, \tilde x) d\pi_i(x_i, \tilde x) = \int_{\X_1 \times \dots \times \X_K \times \Delta_\Y} m_i h(x_i, T(x, m)) d\pi(\vec x, m), \quad \forall h \in C_b(\X_i \times \X), \]
where $T: \X_1 \times \dots \times \X_K \times \Delta_\Y \rightarrow  \X$ is a Borel measurable map satisfying
\[  T(\vec{x}, m) \in \argmin_{\tilde x \in \X } \sum_{i\in \Y} m_i c(x_i, \tilde x); \]
existence of a Borel measurable map satisfying the above property follows from the assumption on $c$ and standard measurable selection theorems. It follows that $\pi_i \in \Gamma_1(\mu_i)$.

Now, notice that for any $f \in \F_{\mathrm{all}}$ we have
\begin{align*}
\sum_{i \in \Y} & \int_{\X_i \times \X} \ell(f(\tilde x),i) d\pi_i(x_i, \tilde x) - \sum_{i \in \Y} \int_{\X_i \times \X} c(x_i, \tilde x) d\pi_i(x_i, \tilde x)
\\& = \int_{\X_1 \times \dots \times \X_K \times \Delta_\Y} \left(\sum_{i \in \Y} m_i \ell(f(T(\vec{x},m)), i) - \sum_{i \in \Y} m_i c(x_i, T(\vec x, m))\right) d\pi(\vec x, m) 
\\& = \int_{\X_1 \times \dots \times \X_K \times \Delta_\Y}\left(\sum_{i \in \Y} m_i \ell(f(T(\vec x , m)), i) - c_{\Y}(\vec x, m)\right) d\pi(\vec x, m) 
\\& \geq \int_{\X_1 \times \dots \times \X_K \times \Delta_\Y} \left(\ell_\Y(m) - {c}_{\Y}(\vec x, m)\right) d\pi(\vec x, m).
\end{align*}
Since this is true for every $f \in \F_{\mathrm{all}}$ and since $\pi \in \G$ was arbitrary, we deduce \eqref{eq:AuxMainProof}. This completes the proof of the equality $\eqref{dual form}=\eqref{eqn:ATGeneralLoss}$ for $\F=\F_{\mathrm{all}}$ and $\G = C_b(\X)$.

The final part in the theorem follows from the above argument given that if $f^*$ is \textit{assumed} to be Borel measurable, then $(\phi_1^*, \dots, \phi_K^*),f^*$ would be feasible (and in turn optimal) for \eqref{eq:ATGeneralLoss_Dual_0}, following the proof of Proposition \ref{proposition:duality}. In addition, since $C_b(\X) \subseteq \G$, it follows that the value of \eqref{eq:ATGeneralLoss_Dual_0} for $\G$ is smaller than or equal to the value of \eqref{eq:ATGeneralLoss_Dual_0} for $C_b(\X)$. Hence, the value of \eqref{eq:ATGeneralLoss_Dual_0} with $\G$ is also equal to the value of \eqref{eqn:ATGeneralLoss} with $\F=\F_{\mathrm{all}}$. From the discussion in the proof of Proposition \ref{Prop:First}, it follows that $f^*$ is optimal for \eqref{eqn:ATGeneralLoss} with $\F = \F_{\mathrm{all}}$.
\end{proof}

\nc

\subsection{Proofs of Section \ref{sec:Examples}}
\label{sec:ProofsConcreteLosses}

We first state a general result that will be useful in the discussion of the examples considered in section \ref{sec:Examples}.

\begin{lemma}
Suppose that $\ell$ satisfies Assumption \ref{assump:LossFunction}. Then $v^* \in \Delta_\Y$ is a minimizer of the problem 
\begin{equation}
\min_{v \in \Delta_\Y}  \max_{m \in \Delta_\Y} \sum_{i\in \Y}( \ell(v, i )  - \phi_i^c(\tilde x)) m_i
\label{eqn:OptimalClassifierEquation}
\end{equation}
if and only if there exist $\lambda_v, \lambda_m \in \R$, $\gamma_v, \gamma_m \in \R^K_+$, and $m^* \in \Delta_\Y$ such that
\begin{enumerate}    
    \item $\vec{0}_K \in \partial_v \vec \ell(v^*)  m^* + \{\lambda_v \vec 1_K -\gamma_v\}  $ \label{enum condition v opt 1}
    \item $\gamma_v \odot v^* = \vec{0}_K$ \label{enum condition v opt 2}
    \item $\vec \ell(v^*) - \vec \Phi^c + \lambda_m \vec 1_K + \gamma_m =\vec{0}_K  $ \label{enum condition m opt 1}
    \item $\gamma_m  \odot m^* = \vec{0}_K $ \label{enum condition m opt 2},
\end{enumerate}
where  $\vec \ell(v) = (\ell(v,i))_{i \in \Y}$, $\Phi^c=(\phi_i^{c}(\tilde x) )_{i \in \Y}$, and $\partial_v  \vec \ell$ is the matrix whose columns are the subdifferentials of the functions $\ell(\cdot, i)$. 
\label{lemma:optimal_K}
\end{lemma}

\begin{proof}[Proof of Lemma \ref{lemma:optimal_K}]
The proof follows from the characterization of the optimal solution to a minimax problem on a compact set. Indeed, Ky Fan's minimax theorem (see Theorem 4.36 in \cite{clarke2013functional}) implies that there is no duality gap in \eqref{eqn:OptimalClassifierEquation} and that the minimization and maximization operations can be applied in any order. The desired result follows from the Kuhn-Tucker conditions under the Slater qualification condition and the subdifferential characterization of the optimal. 
\end{proof}

\subsubsection{Cross-entropy loss}
\label{Subsec: Cross-entropy loss}

\begin{proof}[Proof of Corollary \ref{cor:CrossEntropyOptimizer}]
Thanks to Theorem \ref{thm:main} we may focus on finding solutions to \eqref{eqn:OptimalClassifierEquation} for the choice $\ell=\ell_{\mathrm{ce}}$ and $\phi_i=\phi_i^*$. Now, if we take $v^*$ as in \eqref{classifier cross-entropy pure} and consider $m^*= v^*$, $\gamma_v= \gamma_m = \vec{0}_K$, $\lambda_v =1$, and $\lambda_m = - \log \left( \sum_{i \in \Y} \exp(-\phi_i^c(\tilde x)) \right)$ it is straightforward to verify conditions 1-4 in Lemma \ref{lemma:optimal_K}. This implies the optimality of $v^*$. 
\end{proof}

\begin{remark}
    We note that the form for the solution of \eqref{eqn:OptimalClassifierEquation} for the cross-entropy loss can be \textit{derived} directly from the conditions 1-4 in Lemma \ref{lemma:optimal_K}. Indeed, due to the shape of the cross-entropy loss function (in particular, the fact that $\lim_{t \rightarrow 0^+}\ell_{\mathrm{ce}}(t,i)= \infty$ ), an optimal $v^*$ for \eqref{eqn:OptimalClassifierEquation} must lie in the interior of $\Delta_\Y$. Therefore, by condition 4 in Lemma \ref{lemma:optimal_K} we must have $\gamma_m = 0$. In turn, condition 3 implies that \eqref{classifier cross-entropy pure} is the only possible form that an optimizer can have.  
\end{remark}

\begin{proof}[Proof of Corollary \ref{cor:CrossEntropyBarycenter}]
In this case, $\beta(t)= -\log(t)$ and a direct calculation reveals that the function $\varphi$ in \eqref{eqn:Varphi} becomes
\[ \varphi(s)= s\log(s)  -s.\]
In addition, for $\{ \tilde{\mu}_i \}_{i \in \Y}$ for which $\sum_{i \in \Y} C(\mu_i, \tilde \mu_i)$ is finite we must have $\sum_{i \in \Y} \tilde \mu_i(\X) =\sum_{i \in \Y} \mu_i(\X)=1 $. The desired result follows from these two facts.
\end{proof}

\begin{proof}[Proof of Corollary \ref{cor:CrossEntropyDual}]
It suffices to show that, up to a change of variables, the constraint 
\begin{equation}
   0 \geq \sup_{m_A \in \Delta_{A}} \left\{ \sum_{i \in A} m_i\phi_i(x_i)  +  \ell_A(m_A) - c_A(x_A , m_A)      \right\}, 
   \label{eq:ConstarintAux}
\end{equation}
for a given $x_A =\{ x_i \}_{i \in A} \in  \supp(\mu_A)$ and $A \subseteq  \Y $, reduces to the constraint in \eqref{dual_CE}.

To see this, let us start by denoting by $ \mathcal{A}$  the collection of subsets $A'$ of $A$ such that $\bigcap_{i \in A'} B_\veps(x_i) \not = \emptyset$. With this notation in hand, observe that if $m_A \in \Delta_{A}$ is such that the set $\{ i \in A \text{ s.t. } m_i >0 \}$ is not in $\A$, then $c_A(x_A, m_A) = \infty$. On the other hand, if the set $\{ i \in A \text{ s.t. } m_i>0 \}$ is contained in $\A$, then $c_A(x_A, m_A) =0$.

Observe, also, that for any $m_A \in \Delta_A$ we have
\[ \ell_A(m_A) = \inf_{v \in \Delta_{\Y}} \sum_{i \in A} \ell_{\text{ce}}(v,i) m_i =\inf_{v \in \Delta_{A}} \sum_{i \in A} \ell_{\text{ce}}(v,i) m_i = - \sum_{i \in A} \log(m_i) m_i,  \]
which follows from the fact that for any $v \in \Delta_A$ we have
\[ \sum_{i \in A} \ell_{\mathrm{ce}}(v, i) m_i = \mathrm{KL}(m_A| v) - \sum_{i \in \A} \log(m_i) m_i.  \]

Combining the above observations, we deduce that
\begin{align*}
\sup_{m_A  \in \Delta_{A}} & \left\{ \sum_{i \in A} m_i\phi_i(x_i)  +  \ell_A(m_A) - c_A(x_A , m_A)      \right\}  
    \\& = \sup_{A' \in \A} \sup_{m_{A'} \in \Delta_{A'}} \left\{ \sum_{i \in A'} m_i \phi_i(x_i) - \sum_{i \in A'} \log(m_i) m_i  \right\}
    \\& = \sup_{A' \in \A}  \log(\sum_{i \in A'} \exp(\phi_i(x_i)) ). 
\end{align*}
Therefore, the constraint \eqref{eq:ConstarintAux} is equivalent to 
\[ \sum_{i \in A'} \exp(\phi_i(x_i)) \leq 1 , \quad \forall A' \in \A.    \]
The desired result easily follows after applying the change of variables $\psi_i(x_i) := \exp(\phi_i(x_i))$. 
\end{proof}

\subsubsection{$\alpha$-logarithmic loss}
\label{Subsec: Proof of results alpha fair packing loss}

\begin{proof}[Proof of Corollary \ref{cor:OptimalClassAlpha}]
We split the proof into two cases.

\textbf{Case 1:} When $\alpha>1$, the situation is very similar to the cross-entropy case and a solution $v^*$ for \eqref{eqn:OptimalClassifierEquation} must lie in the interior of $\Delta_\Y$. From this fact, we can deduce that an optimal $v^*$ must take the form $\ref{classifier alpha>1}$. As an alternative argument, consider $v^*$ as in \eqref{classifier alpha>1} and set $m^*_i:= (v_i^*)^\alpha/ (\sum_{j \in \Y} (v_j^*)^\alpha)$, $\gamma_v= \gamma_m =\vec{0}_K$, $\lambda_v =1/\sum_{j \in \Y} (v_j^*)^\alpha$, and $\lambda_m = - Z(\tilde x) $ to directly verify 1-4 in Lemma \ref{lemma:optimal_K}.

\textbf{Case 2:} When $\alpha \in [0,1)$, let $v^*$ be as in \eqref{classifier alpha<1} and set $m^*_i:= (v_i^*)^\alpha/ (\sum_{j \in \Y} (v_j^*)^\alpha)$ for every $i \in \Y$. Also, let $\lambda_v =1/\sum_{j \in \Y} (v_j^*)^\alpha$ and $\lambda_m = -Z(\tilde x) $. Finally, set
\[\gamma_v^i := \lambda_v  \]
for those $i$ for which $v_i^* =0$, and set $\gamma_v^i=0$ otherwise. Likewise, define
\[ \gamma_m^i := \phi_i^c(\tilde x) + Z(\tilde x) - \frac{1}{1-\alpha} \]
for those $i$ for which $v_i^*=0$, and set $\gamma^i_m =0$ when $v_i^*>0$.

Note that $\gamma_v^i $ is greater than or equal to zero for all $i \in \Y$ because $\lambda_v > 0$. On the other hand, $\gamma_m^i \geq0$ for all $i \in \Y$ thanks to the fact that $v_i^*=0 $ if and only if $ -\phi_i^c(\tilde x) - Z(\tilde x) \leq -\frac{1}{1-\alpha}$, as can be easily verified from the properties of $\log_\alpha$ for $\alpha\in [0,1)$. 

Using the fact that $\log_\alpha(0)= -\frac{1}{1-\alpha}$, we can directly verify that 1-4 in Lemma \ref{lemma:optimal_K} hold for the above choices of parameters. We deduce that $v^*$ is optimal for \eqref{eqn:OptimalClassifierEquation}.
\end{proof}

\begin{proof}[Proof of Corollary \ref{cor:alphaFairBarycenter}]

In this case $\beta(t)= -\log_\alpha(t)$ and a direct calculation reveals that the function $\varphi$ in \eqref{eqn:Varphi} becomes
\[ \varphi(s)= s\left( \frac{s^{q-1} -1}{q-1} \right)  -s,\]
for $\alpha >0$ and $\alpha \not = 1$. When $\alpha =0$, we have $\varphi(s) = -s$ for $s\leq 1$ and $\varphi(s)=\infty$ for $s >1$. In addition, for $\{ \tilde{\mu}_i \}_{i \in \Y}$ for which $\sum_{i \in \Y} C(\mu_i, \tilde \mu_i)$ is finite we must have $\sum_{i \in \Y} \tilde \mu_i(\X) =\sum_{i \in \Y} \mu_i(\X)=1 $. The desired result follows from these two facts.
\end{proof}

\begin{proof}[Proof of Corollary \ref{cor:AlphaFairDual}]

As for the cross-entropy loss, recall that if the set of $i \in A$ for which $m_i >0$ is contained in $\A$ (the collection of subsets $A'$ of $A$ such that $\bigcap_{i \in A'} B_\veps(x_i) \not = \emptyset$), then $c_A(x_A, m_A) =0$, while  $c_A(x_A, m_A) = \infty$ otherwise. Thus, as before, we can focus on the case $x_A \in \spt(\mu_A)$ s.t. $\bigcap_{i \in A} B_\veps(x_i) \not = \emptyset$, where we get
\[ \sup_{m_A  \in \Delta_{A}}  \left\{ \sum_{i \in A} m_i\phi_i(x_i)  +  \ell_A(m_A) - c_A(x_A , m_A)      \right\}   = \sup_{m_A  \in \Delta_{A}}  \left\{ \sum_{i \in A} m_i\phi_i(x_i)  +  \ell_A(m_A) \right\} .\]

Now, observe that the right hand side of the above expression can be rewritten as 
\begin{equation}
\sup_{m_A \in \Delta_A} \inf_{v \in \Delta_A} \sum_{i \in A}  ( \ell_{\alpha}(v, i)  +   \phi_i(x_i)  ) m_i,  
\label{eqn:AuxLogalpha}
\end{equation}
using the fact that $\ell_A(m_A)= \inf_{v \in \Delta_\Y} \sum_{i \in A} \ell_\alpha(v,i)m_i = \inf_{v \in \Delta_A} \sum_{i \in A} \ell_\alpha(v,i)m_i$. We can directly adapt the analysis in the proof of Corollary \ref{cor:OptimalClassAlpha} and deduce that the pair $(m^*, v^*)$ defined according to 
\[ v_i^*= \begin{cases}  \exp_\alpha \left( \max \left\{ \phi_i(x_i) - Z(x_A), -\frac{1}{1-\alpha}  \right\} \right) , & \text{ if } \alpha  \in [0,1),\\ \exp_\alpha(\phi_i(x_i) - Z(x_A)), & \text{ if } \alpha >1, \end{cases}\]
\[ m_i^* = \frac{(v_i^*)^\alpha}{\sum_{j \in A} (v_j^*)^\alpha},\]
for $i \in A$, is a saddle for the max-min problem \eqref{eqn:AuxLogalpha}; in the above, $Z(x_A)$ is a normalization that guarantees that $v^* \in \Delta_A$. The value of \eqref{eqn:AuxLogalpha} can thus be written as
\[ \sum_{i\in A} (-\log_\alpha(v_i^*) + \phi_i(x_i) ) m_i^*= \frac{\sum_{i\in A} (-\log_\alpha(v_i^*) + \phi_i(x_i) ) (v_i^*)^\alpha}{ \sum_{j \in A} (v_j^*)^\alpha   },  \]
and requiring for \eqref{eqn:AuxLogalpha} to be less than or equal to zero is in turn equivalent to the condition
\begin{equation}
\sum_{i \in A}  (- \log_\alpha(v_i^*)  +\phi_i(x_i))(v_i^*)^\alpha \leq 0.
\label{eqn:NewConditionAlpha}
\end{equation}
The subsequent analysis is split into two cases.

\textbf{Case 1:} In case $\alpha >1$, plugging the formula for $v^*$ in condition \eqref{eqn:NewConditionAlpha} we deduce $ Z(x_A) \sum_{i \in A} v_i^* \leq 0$, which is equivalent to $Z(x_A) \leq 0$. In turn, this condition is equivalent to 
\[ 1= \sum_{i \in A} \exp_\alpha(\phi_i(x_i) - Z(x_A)) \geq  \sum_{i \in A} \exp_\alpha(\phi_i(x_i)),   \]
thanks to Remark \ref{rem:Porpertiesexp+Alpha}. 

We conclude that problem \eqref{dual form} is equivalent to 
 \begin{equation*}
    \begin{aligned}
    &\inf_{ \{ \phi_i \}_{i \in \Y} \subseteq \G  }& \quad   -  &\sum_{i \in \Y} \int_\X \phi_i(x_i)  d \mu_i (x_i), \\
    & \qquad \mathrm{s.t.}& & \sum_{i\in A} \exp_\alpha(\phi_i(x_i)) \leq   1 \quad \forall A \subseteq \Y, \quad \forall x_A \in \spt(\mu_A) \quad  \mathrm{ s.t. } \quad  \bigcap_{i \in A} B_\veps(x_i) \not = \emptyset,
      \end{aligned}
\end{equation*}
and after the change of variables $\psi_i = \exp_\alpha(\phi_i) $ we deduce the desired result in the case $\alpha>1$.

\textbf{Case 2:} In case $ \alpha \in [0,1)$, condition \eqref{eqn:NewConditionAlpha} can be equivalently written as
\begin{align*}
   0 & \geq  \sum_{i \in A} (-\max\{ \phi_i(x_i) -Z(x_A), - \frac{1}{1-\alpha}  \}  + \phi_i(x_i) ) (v_i^*)^\alpha 
   \\ & =Z(x_A) \sum_{i \in A \, \text{s.t.} \, v_i^* >0} (v_i^*)^\alpha
   \\& = Z(x_A) \sum_{i \in A} (v_i^*)^\alpha,
\end{align*}
where in the second line we have used the fact that $v_i^*=0$ if and only if $\phi_i(x_i) -Z(x_A) \leq -\frac{1}{1-\alpha}$. Hence, \eqref{eqn:NewConditionAlpha} is equivalent to $Z(x_A) \leq 0$, just as in the $\alpha>1$ case. This condition, in turn, can be seen to be equivalent to
\[ 1\geq   \sum_{i \in A}\exp_\alpha \left( \max \left\{ \phi_i(x_i), -\frac{1}{1-\alpha}  \right\} \right).\]
We conclude that problem \eqref{dual form} is equivalent to 
 \begin{equation*}
    \begin{aligned}
    &\inf_{ \{ \phi_i \}_{i \in \Y} \subseteq \G }& \,   -  &\sum_{i \in \Y} \int_\X \phi_i(x_i)  d \mu_i (x_i), \\
    &\mathrm{s.t.}& & \sum_{i\in A} \exp_\alpha(\max\{ \phi_i(x_i), -\frac{1}{1-\alpha} \}) \leq   1, \, \forall A \subseteq \Y,\, \forall x_A \in \spt(\mu_A)\,  \mathrm{ s.t. } \bigcap_{i \in A} B_\veps(x_i) \not = \emptyset.
      \end{aligned}
\end{equation*}
Now, for any feasible tuple $\{\phi_i\}_{i \in \Y}$ in the above problem, the new tuple $\max\{ \phi_i, -\frac{1}{1-\alpha}  \}$ is feasible and moreover does not worsen the objective function of the original tuple. Hence, we can assume that the $\phi_i$ take values in the domain of $\exp_\alpha$ and then consider the change of variables $\psi_i= \exp_\alpha(\phi_i)$. The desired result follows immediately.  
\end{proof}

\subsubsection{Quadratic loss}

\begin{proof}[Proof of Corollary \ref{cor:OptimalClassQuadratic}]
Thanks to Theorem \ref{thm:main} we may focus on finding solutions to \eqref{eqn:OptimalClassifierEquation} for the choice $\ell=\ell_{Q}$ and $\phi_i=\phi_i^*$.

First, note that, even though $\partial_v \vec \ell(v)  $ is not a diagonal matrix as for the other loss functions already considered, a direct computation provides the explicit form
\[ \partial_v \vec \ell(v)  = - 2 \mathbb{I}_K + 2 v \otimes \vec 1_K.\] 
From this we deduce that, regardless of the value of $v \in \Delta_\Y$, condition 1 in Lemma \ref{lemma:optimal_K} is satisfied with the choices $\lambda_v=0$, $\gamma_v = \vec{0}_K$, and $m =v$. With these choices, condition 2 in Lemma \ref{lemma:optimal_K} is also satisfied. 

On the other hand, condition \ref{enum condition m opt 1} is equivalent to
\[(|v|^2+1)\vec 1_K - 2 v - \Phi^c + \lambda_m \vec 1_K + \gamma_m =\vec{0}_K,\]
or, after simplifications, to
\[ v = \frac{1}{2}\left[ (|v|^2 +1 + \lambda_m)  \vec 1_K + \gamma_m -\Phi^c \right],\]
for some vector $\gamma_m$ with non-negative entries and for a scalar $\lambda_m$. To obtain an explicit form for $v=v^*$, assume, without loss of generality, that $\phi^c_1(\tilde x) \leq \phi^c_2(\tilde x) \leq \ldots \leq \phi^c_K(\tilde x) $. With the usual convention $\min(\emptyset) =  \infty$, let $i^*$ and $c^*$ be given by 
\[ i^*  = K \wedge \min \{ i = 1, \ldots, K \quad  \mathrm{ s.t. } \quad  i \phi^c_{i+1}(\tilde x) -   \sum_{j=1}^{i}\phi_j^c(\tilde x) >2    \},\]
and 
\[c^* = \frac{1}{i^*} (2+ \sum_{i=1}^{i^*} \phi^c_i). \]
Let $v^*$ be defined as
\[ v_i^* := 
\begin{cases}
    \frac{1}{2} (c^*  - \phi^c_i),  & \text{ if } i\leq  i^*, \\
    0, & \mathrm{else},
\end{cases} 
\]
which can be seen to satisfy $v^* \in \Delta_\Y$. That the coordinates of $v^*$ sum to one is straightforward from the definition of $c^*$ and $i^*$. The fact that $v_i^*\geq 0$ for $i \leq i^*$ follows from the definition of $i^*$ and the fact that $\phi_i^c$ is non-decreasing in $i$. Indeed, if for the sake of contradiction we assumed that $v_i^*<0$ for some $i\leq i^*$, then we would contradict the definition of $i^*$.

Finally, we may take
\[ \gamma_m^i = \begin{cases}
    0, & \mathrm{ if } \, {i \leq i^*},
     \\\phi_i^c - c^*  & \mathrm{ else}, 
\end{cases}\]
and $\lambda_m = c^*-|v^*|^2-1$ and with all these choices verify conditions 3-4 in Lemma \ref{lemma:optimal_K}; note that, from the definition of $c^*$ and the fact that $\phi_i^c$ is non-decreasing in $i$, it follows that $\gamma_m$ indeed has non-negative entries. We conclude the desired result.
\end{proof}

\section{Applications}
\label{subsec:Lowerbounds}

It is straightforward to show that the $\alpha$-logarithmic losses in \eqref{def:AlphaLoss} are monotonically ordered according to
\begin{equation} \ell_{\alpha}(v, i) \leq \ell_{\alpha'}(v,i) \leq \ell_{\mathrm{ce}}(v,i), \quad v \in \Delta_\Y, \, i \in \Y, 
\label{eqn:CompAlphas}
\end{equation}
for all $ 0 \leq \alpha \leq \alpha' < 1$. Thanks to this and \eqref{eqn:LowerBound}, the learner-agnostic lower bounds for a smaller $\alpha$ are valid lower bounds on the adversarial risk of a model trained with the loss function $\ell_{\alpha'}$ for a larger $\alpha'$. In particular, solving the agnostic adversarial robustness problem with the 0-1 loss provides a lower bound for the adversarial risk of a model trained with the cross-entropy loss. 

In what follows, we illustrate in concrete experimental settings the possible gains of using the tighter bounds that our theoretical results motivate. The code used in our experiments is available at\\ \href{https://github.com/camgt/dual_adversarial_multidim}{https://github.com/camgt/dual\_adversarial\_multidim}.

\subsection{A synthetic example}

To demonstrate the practical implications of our theoretical results, we first consider an adversarially robust classification problem in a simple synthetic setting. We select $\mc X \subset \R^2$, $\mc Y = \{0,1,2\}$, and let $\mu_i$ be concentrated on 20 points sampled from $\mc N(m^i,  \mathbb{I}_2)$, where $m^i$ is one of $ (-2,2),(2,2), (-2,-2);$ an illustration is presented in Figure \ref{fig:points and adversarial risk - synthetic}. We consider the 0$-\infty$ cost function defined in \eqref{eqn:0inftyCost} using the Euclidean distance. We solve the dual problem \eqref{dual_alpha fair} using Python and the CVXOPT library for convex and linear optimization. 

The resulting adversarial risk is shown in Figure \ref{fig:points and adversarial risk - synthetic}. The plots illustrate, as expected, that risk increases with the adversarial budget $\veps$. Indeed, as the adversarial budget increases, points can increasingly interact with other classes. With sufficient adversarial budget, all points can be perturbed into other classes, resulting in \emph{complete confusion}. In that regime, the risk approaches $-\log_\alpha(1/3)$, which corresponds to the risk associated to a uniform distribution over the three classes, thanks to the fact that all the $\mu_i$ have the same number of points. Another aspect clearly illustrated in Figure \ref{fig:points and adversarial risk - synthetic} is that the adversarial risk increases with $\alpha$, in accordance with \eqref{eqn:CompAlphas}.

It is worth noting that there is a clear reduction in complexity when considering dual problems instead of direct adversarial attacks in the primal problem. Indeed, the original (primal) adversarial risk minimization problem would involve searching for a solution in the space of all couplings supported on balls of radius $\veps$ around the original clean data. In contrast, the dual problem in the discrete case requires us to solve only for the dual functions evaluated at the points in the support of the starting measures.

Furthermore, we observed that initializing the solver for a given $\alpha$ with a small perturbation of the solution from a previous $\tilde \alpha<\alpha$ significantly accelerates convergence\footnote{Specifically, we initialize the search with $ (1-\vartheta) \psi^{*}_{\tilde \alpha}  + \vartheta \vec{1}$ for small $\vartheta>0$. The rationale for adding this perturbation is that it always produces points in the interior of the feasible region and improves stability in the search.}, especially when leveraging the sparsity induced by $\alpha<1$ (recall Remark \ref{rem:QualitativeBehavalpha}). This is particularly useful when dealing with non-sparse losses such as cross-entropy. Additional efficiency could be achieved by more intensively distributing some computations, as highlighted in \cite{Diakonikolas_Fair_2020}.

Concerning the classifiers, we illustrate in Figure \ref{fig:classifiers - synthetic}
the optimal classifiers evaluated at the points in the supports of the original distributions $\mu_i$\footnote{Although we limit ourselves to the  points in the domain of $\mu$, classifiers can be computed for other points within a ball of radius $\veps$ from any point in any of the supports of the $\mu_i$.} for the cases $\alpha \in \{0,1\}$ obtained from  \eqref{classifier cross-entropy pure} and \eqref{eq:RobustClassif0-1}. Classifiers respond to the expected degree of confusion, with clearer classification (i.e., higher values) for points within a group as they are farther away from the boundary between groups. We have highlighted significant differences (larger than 0.1) between the optimal classifiers using the 0-1 and cross-entropy losses: All of these appear near the boundary and show that cross-entropy loss seems to give greater importance to the own label of the given points.

\begin{figure}
    \centering
    \includegraphics[width=0.45\linewidth]{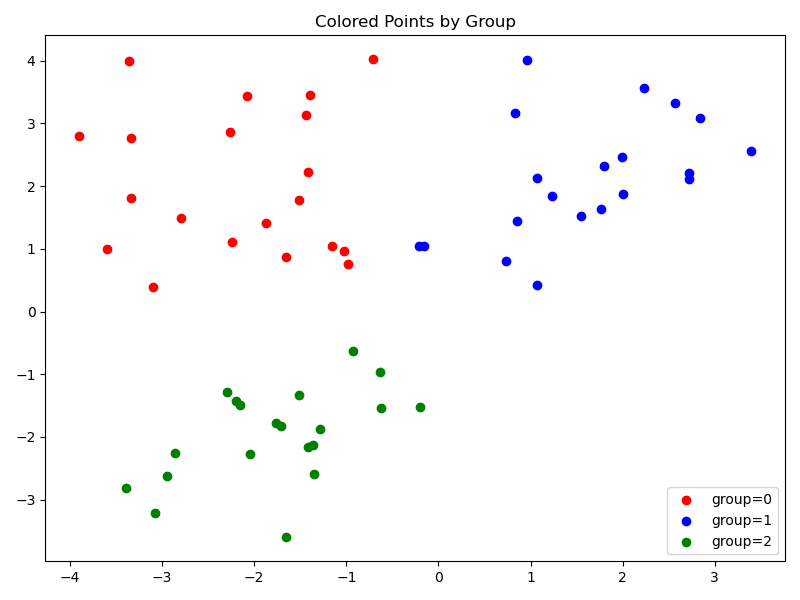}
    \includegraphics[width=0.45\linewidth]{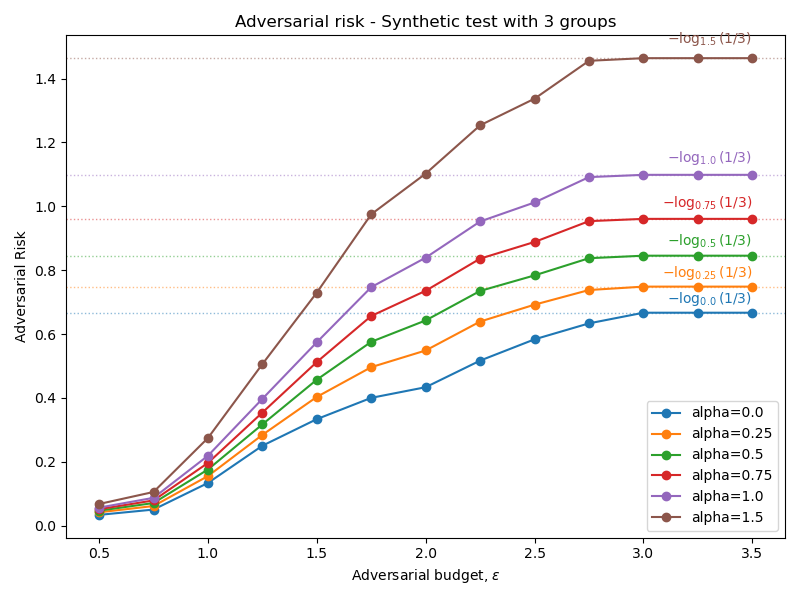}    
    \caption{\textit{Left:} Position of masses for initial measures $(\mu_i)_{i=0,1,2}$; \textit{Right:} Adversarial risk for different $\alpha$. As expected,  plots are monotonic with respect to the adversarial budget, and converge to the risk of full confusion between labels. Notice, also, that plots are  monotonic in $\alpha$ for a fixed budget.    }
    \label{fig:points and adversarial risk - synthetic}
\end{figure}

\begin{figure}
    \centering
    \includegraphics[width=0.8\linewidth]{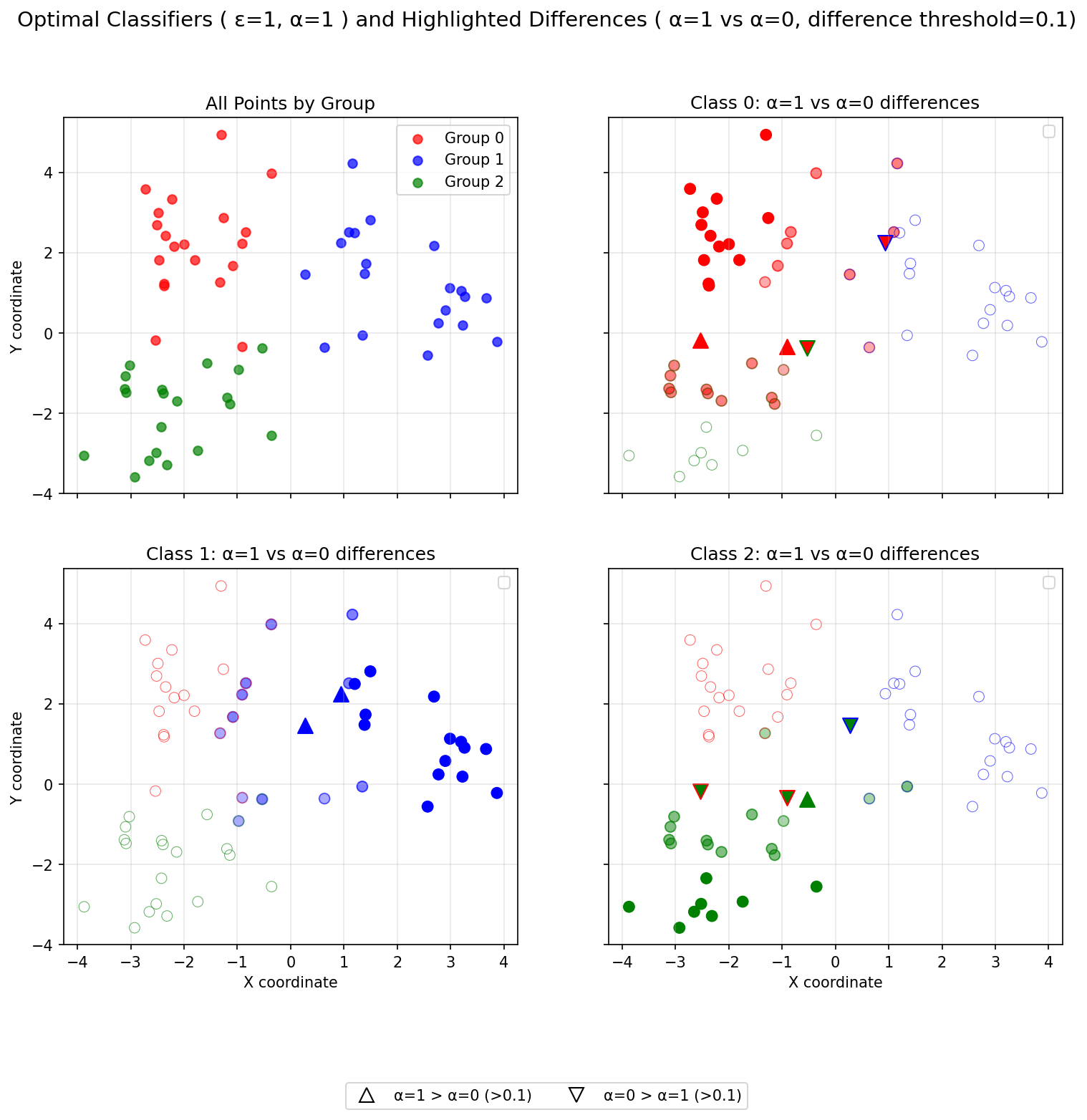}
    \caption{Top left: original data points. Remaining subplots: optimal classifier for each group in the case $\veps =1,\alpha =1$ (i.e. cross-entropy). The value is represented in terms of opaqueness of the interior (higher value, higher opaqueness). The original group is represented by the edge color. Arrows highlight significant differences ($>0.1$) with optimal classifier with same adversarial budget but $\alpha=0$ (0-1 loss). The direction of the arrow indicates the sign of this difference.}
    \label{fig:classifiers - synthetic}
\end{figure}

\subsection{Application to an MNIST sample}

For a more realistic view of the applicability of our results, we turn to the adversarially robust classification of a sample of MNIST images. In this example, $\mc X = \R^{784}$ and we consider four groups corresponding to numbers $1,4,7$ and 9, with 50 images per class.  As before, we consider the 0$-\infty$ cost function defined in \eqref{eqn:0inftyCost} using the Euclidean and Chebyshev distances. 

Figure \ref{fig:MNIST risk} shows the results of solving the dual problem for $\alpha \in \{0, 0.75, 1\}$. We can observe a similar behavior as in the synthetic case, with the risk increasing with the adversarial budget and with $\alpha$. Here, we cap the number of groups that can interact in the dual to either 2 or 3 (as suggested in Remark \ref{rem:Truncation}). As expected, allowing for more interactions produces sharper lower bounds. More importantly, from a practical perspective, truncating the number of interactions has a small effect for small adversarial budgets. Observe the difference in shape between the two distance functions.  Indeed, the Chebyshev case has a staircase behavior corresponding to the fact that, in this metric, points tend to cluster around certain distances from each other. Let us remark that, under the Chebyshev distance, we lack information for $\alpha=1$ when the adversarial budget is large, given that the optimizer that we used did not converge in the specified number of iterations. This illustrates the potential advantages of using intermediate values of $\alpha$ in obtaining sharper lower bounds for a problem with cross-entropy loss than those offered by the 0-1 loss. Overall, this example illustrates the relevance of our theoretical results and reinforces the insights from our synthetic tests.

\begin{figure}
    \centering
    \includegraphics[width=0.75\linewidth]{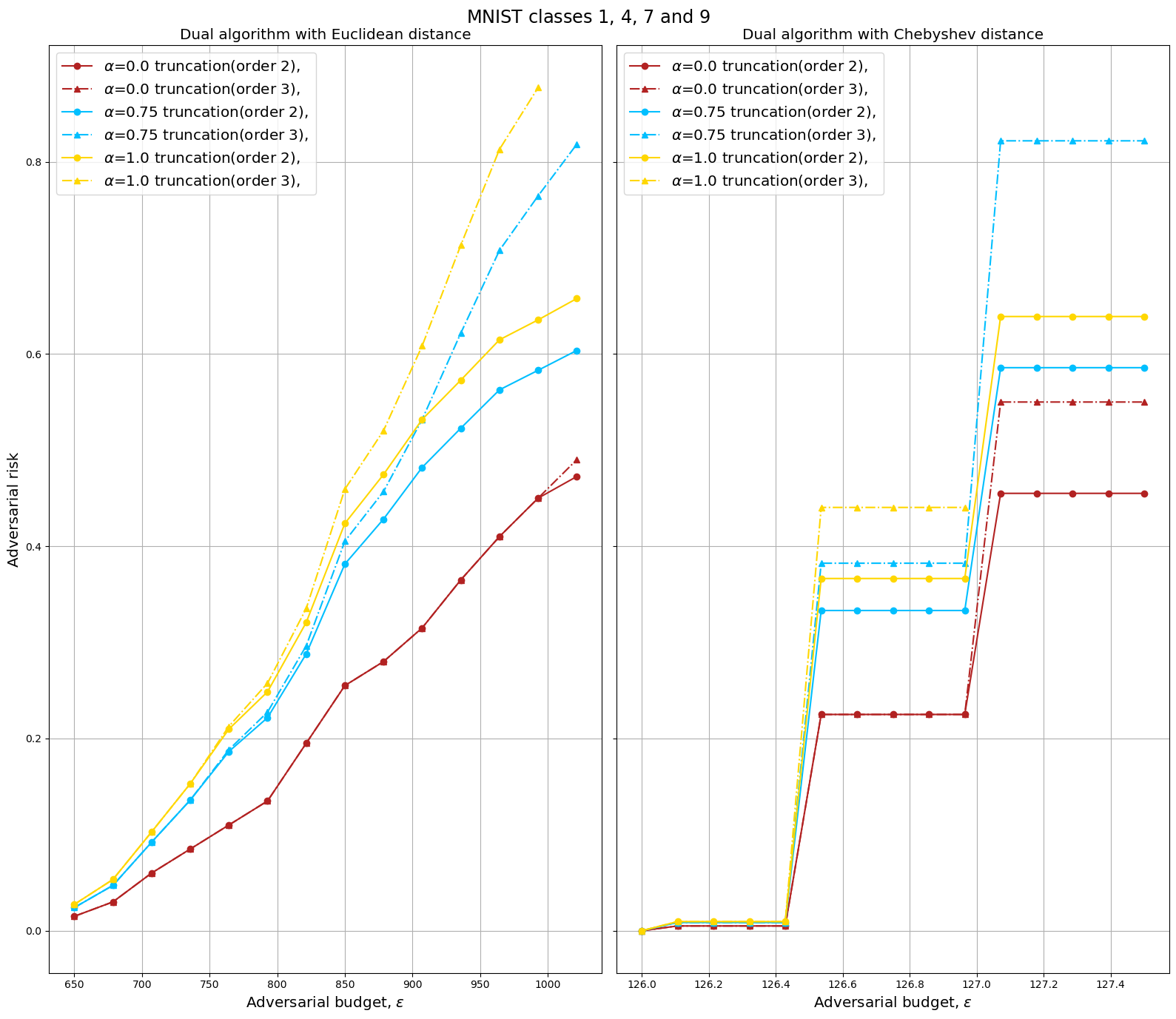}
    \caption{Adversarial risk as a function of adversarial budget for the MNIST test}
    \label{fig:MNIST risk}.
\end{figure}

\section{Conclusions}
\label{sec:Conclusions}

We considered adversarially robust optimization for multiclass supervised learning with general loss functions. We obtained new dual and barycenter formulations for the learner-agnostic adversarial risk minimization problem beyond the 0-1 loss setting, providing in this way sharp lower bounds for adversarial risks under general losses. We studied in detail the quadratic and cross-entropy losses, which are of theoretical and practical interest. We also studied a family of power loss functions that we termed $\alpha$-logarithmic losses, which can be seen to interpolate between the 0-1 and cross-entropy losses. The family of $\alpha$-logarithmic losses has been used in fairness and economics, has good analytical properties, offers theoretical connections to Tsallis entropies and associated divergences through our generalized barycenter results, and provides practical flexibility for classification tasks. Our numerical experiments illustrate the promising practical benefits of our dual formulation, including improved convergence and the potential for distributed optimization techniques. Future work may explore these computational aspects more deeply, including the development of distributed algorithms for the dual problem and the further exploitation of warm-start strategies and sparsity properties to accelerate convergence.

\section*{Acknowledgments}
NGT was supported by the grant NSF-DMS 2236447. The authors would like to thank Jakwang Kim and Matt Werenski for providing the code that facilitated the implementation of our numerical examples. Claude Sonnet 4 was used for code debugging purposes.

\bibliographystyle{siam}
\bibliography{ML}

\appendix

\section{Adversarial training}
\label{app:AT}
As discussed in several papers in the literature (see, e.g., \cite{VarunMuni2,ExistenceSolutionsAT}) there is a close connection between the problem \eqref{eqn:StnadardAT} and the problem \eqref{eqn:ATGeneralLoss} for the cost function $c$ as in \eqref{eqn:0inftyCost}. We point out, however, that some care is needed to rigorously make a statement about this equivalence given that the function 
\[x\mapsto \sup_{\tilde x \in B_\veps(x) } \ell(f(\tilde x ), i), \]
for $B_\veps(x)$ the \textit{closed} ball of radius $\veps$ around $x$,
may not necessarily be Borel measurable if $f$ is only assumed to be Borel measurable. However, if we put these measurability issues aside, we can provide an informal argument suggesting this equivalence. 

First, for the cost $C$ induced by $c_\veps$, it is straightforward to see that $C(\mu_i, \tilde \mu_i) =0$ if and only if there exists $\pi_i \in \Gamma(\mu_i, \tilde \mu_i)$ whose support is contained in the set $\{ (x, \tilde x) \text{ s.t. } d(x, \tilde x) \leq \veps \}$. If the latter condition is not satisfied, then $C(\mu_i, \tilde \mu_i) =\infty$. From this one should formally deduce that \eqref{eqn:ATGeneralLoss} is smaller than \eqref{eqn:StnadardAT}. To motivate the reverse inequality, for a given $f \in \F$ we can formally consider the mapping $x \mapsto T_i(x) \in \mathrm{argmax}_{\tilde x \in B_\veps(x) } \ell( f(\tilde x),i)$ (note that this map may not be Borel measurable if the only thing known about $f$ is that it is Borel measurable). Intuitively, the idea in this construction is to associate the worst possible perturbation to every input $x \in \X$. With these maps, one may then consider the measures $\tilde{\mu}'_i := T_{i \sharp }\mu_i $, $i \in \Y$, and formally get the inequality
\begin{align*}
 \sum_{i \in \Y}   \int_\X \sup_{\tilde x \in B_\veps (x)}\ell(f(\tilde x), i) d \mu_i( x)  & =    \sum_{i \in \Y}   \int_\X \ell(f(\tilde x), i) d\tilde \mu_i'(\tilde x)   - \sum_{i \in \Y} C(\mu_i, \tilde \mu_i') 
 \\& \leq    
\sup _{ \{\tilde{\mu}_i\}_{i \in \mc Y}  } \sum_{i \in \Y}   \int_\X \ell(f(\tilde x), i) d\tilde \mu_i(\tilde x)   - \sum_{i \in \Y} C(\mu_i, \tilde \mu_i),   
\end{align*}
which motivates the reverse relation between \eqref{eqn:ATGeneralLoss} and \eqref{eqn:StnadardAT}.

\section{Additional details in the proof of Theorem \ref{thm:main}}
\label{app:MoreDetails}

\begin{lemma}
Let $c$ be a cost function satisfying Assumption \ref{assump:Cost} and let $\phi_i \in C_b(\X)$. Then $\phi_i^c$ is lower-semicontinuous and hence Borel measurable. 
\label{lem:MeasurabilityCTRansform}
\end{lemma}
\begin{proof}
It is sufficient to prove the result for cost functions $c$ satisfying the compactness and coercivity condition. First, since $\phi_i$ is bounded, it follows that $\phi_i^c$ is bounded from below by a fixed constant. Since $\X$ is a metric space, to prove that $\phi_i^c$ is lower semi-continuous it would suffice to prove that it is sequentially lower-semicontinuous. Toward that aim, suppose that $\tilde x_n \rightarrow \tilde x $, and for each $n \in \N$ let $x_n \in \spt(\mu_i)$ be such that
\[ \phi_i^c(\tilde x_n) + \frac{1}{n} \geq c(x_n, \tilde x_n ) - \phi_i(x_n). \]
If $\liminf_{n \rightarrow \infty} c(x_n, \tilde x_n) =\infty$, we can immediately deduce $\liminf_{n \rightarrow \infty} \phi_i^c(\tilde x_n) \geq \phi_i^c(\tilde x).$ If not, without the loss of generality we can assume that $\sup_{n \in \N} c(x_n, \tilde x_n) <\infty$. Thanks to Assumption \ref{assump:Cost} we can then conclude that, up to a subsequence that we do not relabel (for simplicity), we must have $x_n \rightarrow x$ for some $x$. Since $\spt(\mu_i)$ is always a closed set, the point $x$ must belong to  $\spt(\mu_i)$. We can now use the lower-semicontinuity of $c$ and the continuity of $\phi_i$ to deduce that
\[  \liminf_{n \rightarrow \infty} \phi_i^c(\tilde x_n) \geq \liminf_{n \rightarrow \infty} ( c(x_n, \tilde x_n) - \phi_i(x_n)) \geq c(x, \tilde x) - \phi_i(x) \geq \phi_i^c(\tilde x), \]
completing in this way the proof.
\end{proof}

Next, we present the proof of Proposition \ref{prop:MMOTVersion}. At a high level, the strategy is similar to the proof of the Kantorovich duality theorem appearing in Chapter 1.1.7 in \cite{VillaniBook}. However, the approximation arguments and specific details to make this strategy work are nontrivial adjustments of the ones discussed in \cite{VillaniBook}. These modifications to these arguments are necessary, given that problem \eqref{eqn:MMOTVersion} is not a standard MMOT problem. Below, we restate Proposition \ref{prop:MMOTVersion} in a slightly different way, noticing that the constraint in \eqref{dual form} holds if and only if $\sum_{i \in \Y} m_i\phi_i(x_i) \leq c_\Y(\vec x, m) - \ell_\Y(m)$ for all $x_i \in \X_i, \, i \in \Y$, and $m \in \Delta_\Y$.
\begin{proposition}
\label{prop:WeirdMMOT}
Let $\c: \X_1 \times \dots \times \X_K \times \Delta_\Y \rightarrow \R \cup \{\infty \}$ be defined as
\begin{equation}
   \mathbf{c}(\vec x , m) := c_\Y(\vec{x},m) - \ell_\Y(m),
   \label{eqn:CostTensor}
\end{equation}
for a cost function $c$ satisfying Assumption \ref{assump:Cost} and a loss function $\ell$ satisfying Assumption \ref{assump:LossFunction}.
Then 
\begin{equation}
     \min_{\pi \in \mathfrak{G} } \int_{\X_1 \times \dots \times \X_K \times \Delta_\Y} \mathbf{c}(\vec x , m) d\pi(\vec x, m)
     \label{eqn:MMOTPrimal}
\end{equation}
(recall $\mathfrak{G}$ was introduced in \eqref{eqn:CouplingsWeird})
is equal to 
\begin{equation}
    \begin{aligned}
    &\sup_{ \{ \phi_i \}_{i \in \Y}  \subseteq C_b(\X)  }& \quad  &  \sum_{i \in \Y} \int_\X \phi_i(x_i) d \mu_i (x_i) \\
    & \qquad \mathrm{s.t.}& &    \sum_{i \in \Y} m_i\phi_i(x_i)  \leq \mathbf{c}(x_1, \dots, x_k, m) , \, \forall (\vec{x}, m)  \in \X_1 \times \dots\times \X_K \times \Delta_\Y.
    \end{aligned}
    \label{eqn:DualityRestated}
      \end{equation}
\end{proposition}
\begin{remark}
The functions $\phi_i$ in \eqref{eqn:DualityRestated} belong to $C_b(\X)$ and are thus assumed to be defined in the whole of $\X$. However, as per Tietze's extension theorem, we can equivalently consider $\phi_i \in C_b(\X_i)$, i.e., continuous and bounded functions only defined on $\X_i$. 
\end{remark}

\begin{proof}[Proof of Proposition \ref{prop:WeirdMMOT}]
We split the proof into several steps.

\textbf{Step 0:} We begin with a series of observations. First, note that the cost tensor $\mathbf{c}$ is lower-semicontinuous and bounded from below by a constant. Indeed, the fact that it is bounded from below follows from the fact that $c$ is non-negative and the fact that, thanks to Assumption \ref{assump:LossFunction}, $\ell_\Y$ is bounded above by a positive constant (e.g., by $\max_{i \in \Y}\ell(v_0,i)$).  The lower-semicontinuity of this cost tensor follows from Assumption \ref{assump:Cost} on the cost function $c$ and the fact that $\ell_\Y$ is an upper semi-continuous function (since it is the infimum over a family of continuous functions). Given that $\sum_{i\in \Y}\mu_i$ is a probability measure over $\X$, it follows that the desired strong duality holds if and only if it holds after adding or subtracting a constant to the cost tensor $\mathbf{c}$. Because of this, we will implicitly assume that $\mathbf{c}\geq 0$ throughout the rest of this proof. Finally, note that it is sufficient to prove that $\eqref{eqn:DualityRestated} \geq \eqref{eqn:MMOTPrimal}$, since the reverse inequality follows easily as when analyzing duality in MMOT problems.

\textbf{Step 1:} We will first prove the result under the additional assumptions that the sets $\X_1, \dots, \X_K$ are compact and the cost function $c$ is bounded. We seek to apply the Fenchel-Rockafellar duality theorem (Theorem 1.9 in \cite{VillaniBook}) with a suitable choice of spaces and functions. In particular, we consider the Banach space $E= C_b(\X_1, \dots, \X_K\times \Delta_\Y)$, whose dual $E^*$ is $\M(\X_1 \times \dots \times \X_K \times \Delta_\Y)$ (the space of finite signed Borel measures on $\X_1\times \dots \X_k \times \Delta_\Y$), thanks to the compactness assumption on the sets $\X_i$. Next, we define the (convex) functions $\Theta, \Xi : E \rightarrow \R\cup \{ \infty\} $ according to
\[ \Theta(\Phi):= \begin{cases} 0, & \text{ if } \quad \Phi(\vec x , m)\geq-   \mathbf{c}(\vec{x},m),
\\ \infty, & \text{ else,}\end{cases} \]
\[ \Xi(\Phi):= \begin{cases} \sum_{i \in \Y} \int_{\X_i} \phi_i(x_i) d\mu_i(x_i), & \text{ if } \quad \Phi(\vec x , m)= \sum_{i \in \Y} m_i \phi_i(x_i),   
\\ \infty, & \text{else.} \end{cases} \]
A direct computation reveals that the Fenchel dual of $\Theta$ is
\[ \Theta^*(-\pi) = \sup_{\Phi \in  E  } \{ -\int \Phi d\pi   - \Theta(\Phi) \} = \begin{cases}  \int\mathbf{c}d \pi, & \text{ if } \pi \in \M_+(\X_1 \times \dots \times \X_K \times \Delta_\Y),\\ \infty, & \text{ else}, \end{cases}  \]
because $\mathbf{c}$ is lower semi continuous and non-negative (thus it admits a monotone approximation from below with continuous and bounded functions). Also, $\Xi$'s dual is
\[ \Xi^*(\pi)=   \sup_{\Phi \in  E  }  \{ \int \Phi d\pi   - \Xi(\Phi) \} = \begin{cases}  0, \quad \text{ if }  P_i \pi = \mu_i, \quad \forall i \in \Y, \\ \infty, \quad \text{else}.  \end{cases} \]
The Fenchel-Rockafellar duality theorem gives 
\[ \inf_{\Phi \in E} \{ \Theta(\Phi) + \Xi(\Phi) \} = \max_{\pi \in E^*} \{ - \Theta^*(- \pi) - \Xi(\pi)  \}  , \]
which, after rewriting it, is precisely the desired result under the additional assumptions that the sets $\X_1, \dots, \X_K$ are compact and $c$ is bounded. 
\medskip

\textbf{Step 2:} Next, we relax the assumption that the sets $\X_1, \dots, \X_K$ are compact, but we continue to assume that $c$ is bounded. Let $0<\delta < \frac{1}{4}$. Following the second step in the proof of Theorem 1.3 in \cite{VillaniBook}, we can find compact sets $\X_i^0 \subseteq \X_i$ and positive measures $\mu_i^0$ concentrated on $\X_i^0$ satisfying:
\begin{enumerate}
\item $\mu_i(\X_i \setminus \X_i^0) \leq \delta$ for all $i \in \Y$.
\item $ (1+\delta)\mu_i^0(B) \geq \mu_i(B) \geq (1-\delta) \mu_i^0(B)$ for every Borel subset $B$ of $\X_i^0$ and every $i \in \Y$.
\item $\sum_{i \in \Y} \mu_i^0(\X_i^0) =1.$
\item $\min_{\pi^0 \in \mathfrak{G}_0} \int_{\X_1^0 \times \dots \times \X_K^0 \times \Delta_\Y } \mathbf{c}(\vec{x}, m) d\pi^0(\vec x , m) \geq \eqref{eqn:MMOTPrimal} - \delta,$    
\end{enumerate}
where $\mathfrak{G}_0$ is defined as $\mathfrak{G}$ but with $\mu_i^0$ and $\X_i^0$ in place of $\mu_i$ and $\X_i$, respectively. Applying Step 1 to the measures $\mu_i^0$ (since they are concentrated on the compact sets $\X_i^0$), we can obtain a tuple $\{\phi_i \}_{i \in \Y}$ of functions $\phi_i \in C_b(\X_i^0)$ satisfying
\begin{equation}
    \sum_{i \in \Y} m_i \phi_i(x_i) \leq \mathbf{c}(\vec{x}, m), \quad \forall x_i \in \X_i^0, \, i \in \Y, \, m \in \Delta_\Y,
\label{eqn:FeasibilityAux}
\end{equation}
as well as
\begin{equation}
\sum_{i \in \Y}\int_{\X_i^0} \phi_i(x_i) d\mu_i^0(x_i) \geq  \min_{\pi^0 \in \mathfrak{G}_0} \int_{\X_1^0 \times \dots \times \X_K^0 \times \Delta_\Y } \mathbf{c}(\vec{x}, m) d\pi^0(\vec x , m) - {\delta} \geq \eqref{eqn:MMOTPrimal} -2\delta.
\label{eqn:BetterObjective}
\end{equation}
Our goal now is to use the tuple $\{\phi_i\}_{i \in \Y}$ to construct functions $\{ \tilde \phi_i \}_{i \in \Y}$, with $\tilde \phi_i \in C_b(\X)$ for every $i \in \Y$, that satisfy
\begin{equation}
    \sum_{i \in \Y} m_i \tilde \phi_i(x_i) \leq \mathbf{c}(\vec{x}, m), \quad \forall x_i \in \X, \, i \in \Y, \, m \in \Delta_\Y,
\label{eqn:FeasibilityAux2}
\end{equation}
as well as
\begin{equation}
   \sum_{i \in \Y} \int_\X \tilde \phi_i(x_i) d\mu_i(x_i)  \geq \eqref{eqn:MMOTPrimal}  - C_0\delta,
\label{eqn:LowerBoundDuality}
\end{equation}
for some constant $C_0$ independent of $\delta$. This will be sufficient to deduce the desired duality result under the additional assumption that $c$ is bounded, thanks to the observations we made in Step 0.

We thus focus on constructing the functions $\tilde\phi_i$ mentioned above. This construction is where our argument differs more significantly from the one presented in \cite{VillaniBook}. First, observe that if the tuple $\phi_1, \dots, \phi_K$ is feasible as in \eqref{eqn:FeasibilityAux}, then necessarily
\begin{equation}
 \phi_i(x_i) \leq \lVert c \rVert_{\infty}, \quad \forall x_i \in \X_i^0,\, i \in \Y, 
    \label{eqn:UpperBoundDuals}
\end{equation}
which follows from \eqref{eqn:FeasibilityAux} by just taking $m \in \Delta_\Y$ with $m_i=1$. Next, we claim that we can assume, without the loss of generality, that for every $i \in \Y$ there is $x_i^0 \in \X_i^0$ such that
\[ \phi_i(x_i^0) \geq  - \frac{2K \lVert c \rVert_\infty}{\min_{j \in \Y} \mu_j(\X)} =: D_0 .\]
Indeed, if not, we could take those $i$ for which $\sup_{x_i \in \X_i^0 } \phi_i(x_i) \leq - \frac{2K \lVert c \rVert_\infty}{\min_{j\in \Y} \mu_j(\X)} $ (we will denote this set of $i$ by $\Y_s$)
and consider a number $M_i$ with $M_i \geq  \frac{2K \lVert c \rVert_\infty}{\min_{j \in \Y} \mu_j(\X)}$ such that the sup of $\hat{\phi_i} := \phi_i + M_i$ 
is negative but greater than $ -\frac{K \lVert c \rVert_\infty}{\min_{j \in \Y} \mu_j(\X)}$. For all other $i \in \Y$, we define $\hat{\phi}_i := \phi_i - \lVert c \rVert_\infty$ (in case $\sup_{x_i \in \X_i^0} \phi_i\geq 0$) and set $\hat{\phi}_i = \phi_i$ otherwise. By construction and \eqref{eqn:UpperBoundDuals}, all $\hat{\phi}_i$ are negative and thus satisfy \eqref{eqn:FeasibilityAux}. Furthermore, we see that
\[ \sum_{i \in  \Y_s}M_i\mu_i^0(\X_i^0) - \sum_{i \in  \Y \setminus \Y_s} \lVert c \rVert_\infty \mu_i^0(\X_i^0) >0.  \]
This means that we could replace the $\phi_i$ with the $\hat{\phi}_i$ to obtain a larger value on the left hand side of \eqref{eqn:BetterObjective}. We can now proceed to construct the functions $\tilde \phi_i$ mentioned earlier. 

Let us start with $i=1$ and let $\phi_i'$ be given by
\[  {\phi}'_i(x_i) := \inf_{x_{\Y \setminus \{i \}}, m} \left\{  \frac{1}{m_i} (\c(x_{\Y \setminus \{ i \}}, x_i, m)  - \sum_{j \not = i} m_j \phi_j(x_j)  ) \right\}, \]
where the inf ranges over tuples  $x_{\Y \setminus\{ i\}}$ with $x_j \in \X_j^0$, and $m \in \Delta_y$ such that $m_i >0$. Observe that, thanks to \eqref{eqn:FeasibilityAux}, 
\begin{align*}
    \mathbf{c}(\vec x , m) - \sum_{j \not = i} m_j \phi_j(x_j) & =  \mathbf{c}(\vec x , m) + m_i \phi_i(x_i^0)  -  \sum_{j \in \Y} m_j \phi_j(x_j')
    \\& \geq \mathbf{c}(\vec x , m) - \mathbf{c}(\vec{x}', m) + m_i 
    \phi_i(x_i^0)
    \\& = c_\Y(\vec{x},m) -  c_\Y(\vec{x}',m) + m_i \phi_i(x_i^0)
    \\& \geq c_\Y(\vec{x},m) -  c_\Y(\vec{x}',m) + m_i D_0,
\end{align*}
where in the above we used the tuple $\vec{x}'$ defined as $x_j'=x_j$ for all $j \not =i$, and $x_i' = x_i^0$. Now, observe that for every $\tilde x \in \X$ we have
\begin{align*}
\sum_{j\in \Y} m_j c(x_j, \tilde x) - c_\Y(\vec{x}', m) & \geq   \sum_{j\in \Y} m_j c(x_j, \tilde x) -  \sum_{j\in \Y} m_j c(x_j', \tilde x) 
\\& = m_i( c(x_i, \tilde x) - c(x_i^0, \tilde x)) 
\\ & \geq - m_i\lVert c \rVert_\infty.
\end{align*}
Putting together the above estimates, we deduce 
\[ \frac{1}{m_i} (\c(\vec x, m)  - \sum_{j \not = i} m_j \phi_j(x_j)  ) \geq  D_0 - \lVert c \rVert_\infty =: D_0'. \]
In particular, the function $\phi_i'$ satisfies
\begin{equation}
\phi_i'(x_i) \geq D_0', \quad \forall x_i \in \X. 
\label{eqn:LowerBoundEverywhere}
\end{equation}
Also, from the definition of $\phi_i'$ and \eqref{eqn:FeasibilityAux}, it follows that
\begin{equation}
\phi_i'(x_i) \geq \phi_i(x_i), \quad \forall x_i \in \X_i^0.
\label{eqn:LowerBoundXi0}
\end{equation}
Finally, following a similar argument as in the proof of Lemma \ref{lem:MeasurabilityCTRansform}, we can prove that the function $\phi_i'$ is lower-semicontinuous. Since $\phi_i'$ is bounded from below by the constant $D_0'$, we can find a function $\tilde \phi_i \in C_b(\X)$ bounded from below by $D_0'$ and from above by $\phi_i'$ for which
\begin{equation}
    \int_{\X_i^0} \tilde \phi_i(x_i) d\mu_i(x_i) \geq \int_{\X_i^0}  \phi_i'(x_i) d\mu_i(x_i) -\frac{\delta}{K}.
    \label{eqn:LowerBoundXi0_2}
\end{equation}

Inductively, assuming we have constructed functions $\phi_1', \dots, \phi_{i-1}'$, and $\tilde \phi_1, \dots, \tilde \phi_{i-1} \in C_b(\X)$, we define $\phi_i'$ according to:
\[  {\phi}'_i (x_i) := \inf_{x_{\Y \setminus \{i \}}, m} \left\{  \frac{1}{m_i} (\c(x_{\Y \setminus\{ i\}}, x_i, m)  - \sum_{j < i} m_j \tilde \phi_j(x_j) - \sum_{j > i} m_j \phi_j(x_j)   ) \right\}, \]
where the inf ranges over tuples  $x_{\Y \setminus\{ i\}}$ with $x_j \in \X$ for $j < i$ and $x_j \in \X_j^0$ for $j>i$, and $m \in \Delta_y$ such that $m_i >0$. Repeating the argument as for the case $i=1$, we can verify that $\phi_i'$ satisfies \eqref{eqn:LowerBoundEverywhere} and \eqref{eqn:LowerBoundXi0}. Also, we can find $\tilde{\phi}_i \in C_b(\X)$ bounded from below by $D_0'$ and from above by $\phi_i'$ for which \eqref{eqn:LowerBoundXi0_2} holds.

By definition of $\phi_K'$ and the fact that $\tilde \phi_K \leq \phi'_K $, it follows that the functions $\tilde{\phi}_1, \dots, \tilde{\phi}_K$ satisfy \eqref{eqn:FeasibilityAux2}. Furthermore, 
\begin{align*}
  \sum_{i \in \Y} \int_{\X_i} \tilde{\phi}_i(x_i) d \mu_i(x_i) & =    \sum_{i \in \Y} \int_{\X_i^0} \tilde{\phi}_i(x_i) d \mu_i(x_i) +   \sum_{i \in \Y} \int_{\X_i \setminus \X_i^0} \tilde{\phi}_i(x_i) d \mu_i(x_i) 
  \\&  \geq  \sum_{i \in \Y} \int_{\X_i^0} \phi'_i(x_i) d \mu_i(x_i)-\delta  +   \sum_{i \in \Y} \int_{\X_i \setminus \X_i^0} \tilde{\phi}_i(x_i) d \mu_i(x_i)
  \\& \geq  \sum_{i \in \Y} \int_{\X_i^0} \phi_i(x_i) d \mu_i(x_i)-\delta   - K|D_0'|\delta
   \\& \geq  (1-\delta)\sum_{i \in \Y} \int_{\X_i^0} \phi_i(x_i) d \mu_i^0(x_i)-\delta   - K|D_0'|\delta
   \\& \geq \eqref{eqn:MMOTPrimal} - ( 3 + \lVert c \rVert_\infty + K|D_0'| ) \delta.
\end{align*}
This finishes the proof in this case.
\nc

\textbf{Step 3:} In this final step, we relax the assumption that $c$ is bounded. This, however, is easily accomplished as in step 3 in the proof in \cite{VillaniBook}. For that we consider the cost functions $c^N$ given by 
\[ c^N(x, \tilde x) := \min  \{ c(x, \tilde x), N \}, \quad N \in \N,\]
which are lower-semicontinuous and bounded. We let $c_\Y^N$ and $\mathbf{c}^N$ be defined as $c_\Y$ and $\mathbf{c}$ but with respect to the new cost function $c^N$. It is straightforward to see that $\mathbf{c}^N$ approximates $\mathbf{c}$ monotonically from below. Thanks to Step 2, the duality holds for each $\mathbf{c}^N$ and it remains to follow the same steps as in the last part of the proof of Theorem 1.3 in \cite{VillaniBook} to conclude the desired duality result for $\mathbf{c}$.
\end{proof}

\section{Proof of Theorem \ref{Thm: Dual as generalised barycenter}}
\begin{proof}[Proof of Theorem \ref{Thm: Dual as generalised barycenter}]

From the proof of Theorem \ref{thm:main} we know that \eqref{eqn:ATGeneralLoss} (for $\F=\F_{\mathrm{all}}$) is equal to
\[   \sup _{ \{\tilde{\mu}_i\}_{i \in \mc Y}  }  \inf_{f \in \F_{\mathrm{all}}} \sum_{i \in \Y}   \int_\X \ell(f(\tilde x), i) d\tilde \mu_i(\tilde x)   - \sum_{i \in \Y} C(\mu_i, \tilde \mu_i). \]
It thus suffices to show that the above is equal to \eqref{eqn:ATGeneralLoss_dual_barycentric}. 

To see this, for fixed $\{\tilde \mu_i\}_{i \in \Y}$ we focus on rewriting the minimization problem
\[ \inf_{f \in \F_{\mathrm{all}}}  \sum_{i \in \Y}   \int_\X \beta(f_i(\tilde x)) d\tilde \mu_i(\tilde x).     \]
Let $\Lambda_0= \sum_{i \in \Y} \tilde \mu_i$ and observe that, thanks to the fact that $\beta$ is non-increasing, we have
\begin{align*}
  \inf_{f \in \F_{\mathrm{all}}} \sum_{i \in \Y}   \int_\X \beta(f_i(\tilde x)) d\tilde \mu_i(\tilde x) & =   \inf_{f \in \F_{\mathrm{all}}}    \int_\X  \left(\sum_{i \in \Y} \beta(f_i(\tilde x)) \frac{d\tilde \mu_i}{d\Lambda_0} \right) d\Lambda_0(\tilde x)  
  \\& = \int_{\X}\inf_{v \in \Delta_\Y} \left(\sum_{i \in \Y} \beta(v_i) \frac{d\tilde \mu_i}{d\Lambda_0} \right) d\Lambda_0(\tilde x)
  \\& =\int_{\X}\inf_{v \in \R^K_+ \text{ s.t. } \sum_{i \in \Y} v_i \leq 1 } \left(\sum_{i \in \Y} \beta(v_i) \frac{d\tilde \mu_i}{d\Lambda_0} \right) d\Lambda_0(\tilde x)
  \\&= \int_{\X} \sup_{a> 0}\left( - a \sum_{i \in \Y} \varphi\left(  \frac{1}{a}\frac{d\tilde \mu_i}{d\Lambda_0} \right) - a \right) d \Lambda_0(\tilde x) 
  \\& = \sup_{ a :\X \rightarrow \R_+ \text{ Borel }} \int_{\X} \left( - \sum_{i \in \Y}\varphi\left(  \frac{1}{a(\tilde x)}\frac{d\tilde \mu_i}{d\Lambda_0} \right) - 1  \right) a(\tilde x) d \Lambda_0(\tilde x) 
   \\& = \sup_{ \lambda \in \M_+(\X) \text{ s.t. } \tilde \mu_i \ll \lambda, \, \forall i \in \Y } \int_{\X} \left( - \sum_{i \in \Y}\varphi\left(  \frac{d\tilde \mu_i}{d\lambda} \right) - 1  \right) d \lambda(\tilde x). 
\end{align*}
The desired result now follows.
\end{proof}

\begin{remark}
\label{rem:CrossEntropyLB}
Since the function $\varphi$ in \eqref{eqn:Varphi} can be written as the supremum over linear functions, it is necessarily convex. In addition, by definition of $\varphi$ we have the lower bound
\[ \varphi(s) \geq - \beta(1/K) s  - 1/K, \]
which implies 
\[ \lambda(\X) + \sum_{i \in \Y} \int_{\X}\varphi \left( \frac{d\tilde \mu_i}{d\lambda}\right) d\lambda   \geq -\beta(1/K) \sum_{i \in \Y} \tilde{\mu}_i(\X)     \]
for any $\lambda , \{ \tilde \mu_i\}_{i \in \Y}$. If in addition $\sum_{i\in \Y} C(\mu_i, \tilde \mu_i) < \infty$, we have $\mu_i(\X) = \tilde{\mu}_i(\X)$, and the above bound reduces to $-\beta(1/K)$.

\end{remark}

\nc

\section{$\alpha$-fair packing}
\label{app:AlphaFair}
Given $\alpha \in [0, \infty)$, a general $\alpha$-fair packing problem with linear constraints takes the form
\begin{equation}
    \begin{aligned}
    &\max_{z \in \R^n}  &   &  \sum_{l =1 }^n  U_\alpha(z_l) , \\
    & \: \mathrm{ s.t.} \qquad & & Dz \leq \mathbf{1}, \\
    & \qquad & & 0 \leq z,
   \label{eqn:GeneralAlphaFair}
      \end{aligned}
      \end{equation}
where $U_\alpha(t) =  \log_\alpha(t)$, for $\log_\alpha$ as in \eqref{def:AlphaLoss} for $\alpha\geq 0$ and $\alpha \not =1$, and $\log_1$ the standard natural logarithm $\log$. Problem \eqref{eqn:GeneralAlphaFair} has an economic interpretation. Indeed, we can think of the variable $z=(z_1, \dots, z_n)$ in \eqref{eqn:GeneralAlphaFair} as a possible allocation of a monetary reward among $n$ different parties. When assigned income $z_l$, party $l$ receives \textit{utility} $U_\alpha(z_l)$. The constraint $z \geq 0$ captures the fact that incomes are nonnegative numbers, and the condition $Dz \leq \mathbf{1}$ captures specific additional constraints for the allocation. The goal in \eqref{eqn:GeneralAlphaFair} is to find the allocation of rewards producing the largest possible average utility.

\begin{remark}[On the form of $U_\alpha$]
    In economic theory, the family of functions $U_\alpha = \log_\alpha$ is known as \textit{isoelastic} utility functions. The utility functions in this family have several advantageous properties: they are increasing, concave, and smooth. Further, from an economic point of view,  each function in the family has a constant  \emph{relative risk aversion} (equal to $\alpha$ for $U_\alpha$). The relative risk aversion of the function at a point $x$ is a normalized measure of the curvature of the function, namely
    $$ - \frac{x \partial_{xx} U_\alpha(x)}{\partial_x U_\alpha(x)} = \alpha.$$
    Thanks to the above properties and their computational tractability, this family has been used intensively in utility theory and finance, and has been used to define inequality measures by Atkins in \cite{Atkinson_measurement_1970}.

    In the context of losses in classification, the parameter $\alpha$ can be used to control how harshly the loss penalizes confident misclassifications, which in turn allows users to improve learning for unbalanced distributions (as in \cite{Lin_Focal_2017}, where the power framework appears) or potentially to reduce sensitivity to outliers.     
\end{remark}

As discussed in the main body of the paper, when $\mu = \frac{1}{n}\sum_{l=1}^n \delta_{(x_l, y_l )}$ is an empirical measure, problems \eqref{dual_CE} and \eqref{dual_alpha fair} can be written in the form \eqref{eqn:GeneralAlphaFair}. We provide more details on this assertion. First, we consider the identification
\[ z_l = \psi_{y_l}(x_l).\]
Now, for a given $A \subseteq \Y$ and $x_A \in \spt(\mu_A) $ with $ \bigcap_{i \in A} B_\veps(x_i) \not = \emptyset $, we associate a row in the matrix of constraints $D$ in \eqref{eqn:GeneralAlphaFair}, setting to one those entries corresponding to the variables $\psi_{y_l}(x_l)$ for the $x_l$ in $x_A=\{ x_l \}_{l \in A}$, and setting to zero all other entries. With these identifications, it is clear that, indeed, \eqref{dual_CE} and \eqref{dual_alpha fair} can be written in the form \eqref{eqn:GeneralAlphaFair}.

\section[\texorpdfstring{Recovering the results for the 0-1 loss in \cite{MOTJakwang}}{Recovering the results for the 0-1 loss}]{Recovering the results for the 0-1 loss in \cite{MOTJakwang}}
\label{app:Linear}
In this appendix, we discuss the equivalence between problem \eqref{dual form} for $\ell=\ell_{01}$ (equal to $\ell_\alpha$ for $\alpha =0$) and the problem \eqref{dual_0-1} derived in \cite{MOTJakwang} for the 0-1 loss. We start with a lemma.

\begin{lemma}
Let $\{ a_i \}_{i \in S}$ be a finite collection of real numbers. Then the maximum in the problem
\begin{equation}
\label{eqn:Max0-1LOss}
   \max_{m \in \Delta_{S}} \{ \sum_{i \in S} m_i a_i   - \max_{i \in S} m_i \}  
\end{equation}
is achieved at the uniform measure over a subset of $S$.
\label{lemma:optimal m and s}
\end{lemma}

\begin{proof}
Without the loss of generality, we can assume $S= \{1, \dots, s \}$. We start by observing that the simplex $\Delta_S$ can be written as
\[ \Delta_S = \bigcup_{ p \in \Pi_S } B_p, \]
where $\Pi_S$ denotes the set of permutations of the elements in $S$, and where, for each $p \in \Pi_S$, the set $B_p$ is given by
\[ B_p := \{ m \in \Delta_S \: : \:   m_{p(1)} \geq m_{p(2)} \geq \dots \geq m_{p(s)}  \}. \]
From this, it trivially follows that the max in \eqref{eqn:Max0-1LOss} is reached in at least one of the sets $B_p$. After relabeling the indices if necessary, we can assume that the $B_p$ where the max is reached is the identity permutation, and from now on we use $B$ to denote this set.

Observe that in $B$ the
objective function $\sum_{i\in S} m_i a_i  - \max_{i \in S} m_i$ is a linear function in $m$, since, in $B$, we have $\max_{i \in S} m_i= m_1$. Therefore, the maximum of this objective function over $B$ is achieved at $B$'s extreme points, which, as we discuss next, is the set of points $E= \{ u^1, \dots, u^s \}$, where, for each $r \leq s$, we have
\[ u^r_j := \begin{cases} 1/r, & \quad \text{if } j \leq r, \\ 0, &  \quad \text{else}.  \end{cases}\]
Once we have proved that $E$ is indeed the set of extreme points of $B$ the result will immediately follow.

To prove that $E$ is the set of extreme points of $B$, let us consider an arbitrary element $m$ in $B$, which by definition must satisfy $m_1 \geq m_2\geq \dots \geq m_s$. Let $m_1 \dots, m_t$ denote the nonzero entries of $m$ and set $m_{t+1}:=0$ in case $t=s$. For each $r=1, \dots, t$, let $\alpha_r$ be given by
\[\alpha_r := r(m_r -m_{r+1}),  \]
which is a nonnegative number. A straightforward computation reveals that
\[ \sum_{r=1}^t \alpha_r = \sum_{r=1}^t m_r =1. \]
Moreover, 
\[ m = \sum_{r=1}^t \alpha_r u^r. \]
We have thus shown that any element in $B$ can be written as a convex combination of the elements in $E$. At the same time, it is clear that no element in $E$ can be written as a convex combination of the other elements in $E$. This shows that $E$ is the set of extreme points of $B$. 
\end{proof}

\begin{proposition}
Problem \eqref{dual form} for $\ell=\ell_{\mathrm{01}}$ is equivalent to problem \ref{dual_0-1}. Precisely, the value of \eqref{dual form} is $1$ minus the value of \eqref{dual_0-1}, and $\{ \phi_i\}_{i \in \Y}$ is a solution of \eqref{dual form} if and only if $\{ g_i:= 1+\phi_i \}_{i \in \Y}$ is a solution of \eqref{dual_0-1}. 
\end{proposition}

\begin{proof}
First we prove prove that a tuple $\{ \phi_i\}_{i \in \Y}$ satisfies \eqref{eq:ConstraintFiniteSupport} if and only if
\[  \sum_{i\in A} \phi_i(x_i) \leq   1 -|A| +     c_A(x_A),  \quad \forall A \subseteq \Y, \quad \forall x_A \in \spt(\mu_A), \]
where 
\[ c_A(x_A) := \inf_{\tilde x \in \X} \sum_{i \in A} c(x_i, \tilde x ). \]

Observe that, for a given $m_A \in \Delta_A$,
\[\ell_A(m_A)= \inf_{v \in \Delta_\Y } \sum_{i\in A} \ell(v, i) m_i = \inf_{v \in \Delta_\Y } \sum_{i\in A} (1 - v_i) m_i = 1-\max_{i \in A} m_i. \]
On the other hand, using the definition of $c_A(x_A, m_A)$ we can write
\begin{align*}
\sup_{m_A \in \Delta_{A}} & \left\{ \sum_{i \in A} m_i\phi_i(x_i)  +  \ell_A(m_A) - c_A(x_A , m_A)      \right\} 
\\ & = \sup_{m_A \in \Delta_{A}} \sup_{\tilde x \in \X} \left\{ \sum_{i \in A} m_i\phi_i(x_i)  +  \ell_A(m_A) - \sum_{i \in A} m_i c(x_i, \tilde x)      \right\}.
\end{align*}
Swapping the two sups in the above expression we obtain
\begin{align*}
    \sup_{\tilde x \in \X} &  \sup_{m_A \in \Delta_A} \left\{ \sum_{i \in A} m_i\phi_i(x_i)  +  \ell_A(m_A) - \sum_{i \in A} m_i c(x_i, \tilde x)      \right\} \\ & = \sup_{\tilde x \in \X}  \sup_{m_A \in \Delta_A} \left\{ \sum_{i \in A} m_i\phi_i(x_i)  +  1 - \max_{i \in A} m_i - \sum_{i \in A} m_i c(x_i, \tilde x)      \right\}.
\end{align*}
Now, using Lemma \ref{lemma:optimal m and s} we can restrict the inner sup in the above expression to the $m_A$'s in $\Delta_A$ that are uniform measures over subsets of $A$. In particular, the above is equal to

\begin{align*}
    \sup_{\tilde x \in \X} & \sup_{A' \subseteq A} \left\{  \frac{1}{|A'|}\sum_{i \in A'} \phi_i(x_i)  +  1 - \frac{1}{|A'|} - \frac{1}{|A'|}\sum_{i \in A'} c(x_i, \tilde x)      \right\} \\ & = 
    \sup_{A' \subseteq A} \left\{  \frac{1}{|A'|}\sum_{i \in A'} \phi_i(x_i)  +  1 - \frac{1}{|A'|} - \frac{1}{|A'|}c_{A'}(x_{A'})      \right\}.
\end{align*}
Requiring that the above is smaller than or equal to zero is equivalent to the requirement 
\[ \sum_{i\in A'} \phi_i(x_i) + |A'| -1 - c_{A'}(x_{A'}) \leq 0, \quad \forall A' \subseteq A.    \]

At this stage, it suffices to consider the change of variables $g_i= \phi_i +1$ in order to deduce the equivalence between the two optimization problems. 
\end{proof}

\nc 

\end{document}

%% file: preamble_arxiv.tex
\usepackage{lipsum}
\usepackage{amsfonts}
\usepackage{amsmath}
\usepackage{amsthm}
\usepackage{amssymb}
\usepackage{graphicx}

\usepackage{caption}
\usepackage{subcaption}

\usepackage[nocompress]{cite}

\usepackage{xcolor}

\usepackage{epstopdf}
\ifpdf
  \DeclareGraphicsExtensions{.eps,.pdf,.png,.jpg}
\else
  \DeclareGraphicsExtensions{.eps}
\fi

\usepackage{amsopn}

\usepackage{dsfont}
\usepackage{xcolor}
\usepackage{scalerel,lmodern}
\usepackage[T1]{fontenc}

\newcommand{\ggreen}{\color{mygreen}}
\newcommand{\oorange}{\color{myorange}}

\definecolor{mygreen}{RGB}{0, 102, 0}
\definecolor{myorange}{RGB}{209 , 128, 36}

\newcommand{\nc}{\normalcolor}

\newcommand{\spt}{\mathrm{spt}}

\usepackage{bm}

\usepackage{algorithm}
\usepackage{algpseudocode}
\newtheorem{theorem}{Theorem}
\newtheorem{lemma}[theorem]{Lemma}

\newtheorem{remark}[theorem]{Remark}
\newtheorem{corollary}[theorem]{Corollary}
\newtheorem{proposition}[theorem]{Proposition}

\newtheorem{assumption}[theorem]{Assumption}

\newcommand{\veps}{\varepsilon}

\newcommand{\A}{\mathcal{A}}
\newcommand{\G}{\mathcal{G}}

\newcommand{\F}{\mathcal{F}}
\newcommand{\X}{\mathcal{X}}
\newcommand{\Z}{\mathcal{Z}}
\newcommand{\supp}{\mathrm{spt}}

\newcommand{\R}{\mathbb{R}}
\newcommand{\M}{\mathcal{M}}

\newcommand{\N}{\mathbb{N}}

\newcommand{\mc}{\mathcal}
\renewcommand{\c}{\mathbf{c}}

\newcommand{\Y}{\mathcal{Y}}
\renewcommand{\Z}{\mathcal{Z}}

\DeclareMathOperator*{\argmin}{arg\,min}

\usepackage{enumitem}
\setlist[enumerate]{leftmargin=.5in}
\setlist[itemize]{leftmargin=.5in}

